\documentclass{article}
\usepackage{fullpage}
\usepackage{charter}
\date{}

\usepackage{enumitem}

\usepackage{cite}
\usepackage{xcolor}

\usepackage{xspace}
\usepackage{url}
\usepackage{graphicx}
\usepackage{balance}  
\usepackage{xcolor}
\usepackage{amsmath, amssymb}
\usepackage{amsthm}
\usepackage{bbm}
\usepackage{verbatim}
\usepackage{hyperref}
\usepackage{cleveref}
\usepackage{footnote}
\usepackage{mathtools}
\usepackage{multirow}
\usepackage{subcaption}
\usepackage[utf8]{inputenc}
\usepackage[noend]{algorithmic}
\usepackage{algorithm}

\usepackage{amsmath}
\DeclareMathOperator*{\argmax}{arg\,max}
\DeclareMathOperator*{\argmin}{arg\,min}

\newcommand{\at}[1]{\protect\ensuremath{\mathsf{#1}}\xspace}

\newtheorem{definition}{Definition}
\newtheorem{theorem}{Theorem}

\newtheorem{corollary}{Corollary}
\newtheorem{lemma}{Lemma}
\newtheorem{example}{Example}

\newcommand{\stitle}[1]{\vspace{1ex}\noindent{\bf #1}}

\newcommand{\mg}[1]{\textcolor{purple}{Merve: #1}}
\newcommand{\xc}[1]{\textcolor{blue}{Xu: #1}}

\newcommand{\update}[1]{#1}

\newcommand{\figfmt}{pdf}

\begin{document}

\title{Nearest Neighbor Classifiers over Incomplete Information:\\
From Certain Answers to Certain Predictions \footnote{The first two authors contribute equally to this paper and are listed alphabetically.}
}

\author{
Bojan Karlaš$^{*,\dagger}$,
Peng Li$^{*,\ddagger}$,
Renzhi Wu$^{\ddagger}$,
Nezihe Merve Gürel$^{\dagger}$, \\
Xu Chu$^{\ddagger}$,
Wentao Wu$^{\mathsection}$,
Ce Zhang$^{\dagger}$ \\
\small $^\dagger$ETH Zurich, $^\ddagger$Georgia Institute of Technology, $^\mathsection$Microsoft Research
}

\maketitle

\begin{abstract}
Machine learning (ML) applications have been thriving recently, largely attributed to the increasing availability of data.
However, inconsistency
and incomplete information are ubiquitous in real-world datasets, and their impact on ML applications remains elusive.
In this paper, we present a formal study of this impact by extending the notion of \emph{Certain Answers for Codd tables}, which has been explored by the database research community for decades, into the field of machine learning.
Specifically, we focus on classification problems and propose the notion of \emph{
``Certain Predictions'' (CP)} --- a test data example can be \textit{certainly predicted (CP'ed)} if 
\emph{all} possible classifiers trained on top of all possible worlds
induced by the incompleteness of data
would yield the \emph{same} prediction.
We study two fundamental CP queries: (Q1) \emph{checking query} that determines whether a data example can be CP'ed; and (Q2) \emph{counting query} that computes the number of classifiers that support a particular prediction (i.e., label).
Given that general solutions to CP queries are, not surprisingly, hard without assumption over the type of classifier, we further present a case study in the context of nearest neighbor (NN) classifiers, where efficient solutions to CP queries can be developed ---
we show that it is possible to answer both queries 
in {\em linear or polynomial time}
over {\em exponentially}
many possible worlds.
We demonstrate one example
use case of CP in the important application of ``data cleaning for machine learning (DC for ML).'' We show that our proposed CPClean approach built based on CP can often significantly outperform existing techniques in terms of classification accuracy with mild manual cleaning effort.
\end{abstract}






\section{Introduction}

Building high-quality 
Machine learning (ML) applications 
often hinges on the availability of
high-quality data.
However, due to noisy inputs from manual data curation or inevitable errors from automatic data collection/generation programs, in reality, data is unfortunately seldom clean.
Inconsistency and incompleteness are ubiquitous in real-world datasets, and therefore can have an impact on ML applications trained on top of them.
In this paper, we focus on the question:
\textit{Can we reason about the 
impact of data incompleteness on the 
quality of ML models trained over
it?}

\begin{figure}[t]
\centering
\includegraphics[width=0.6\textwidth]{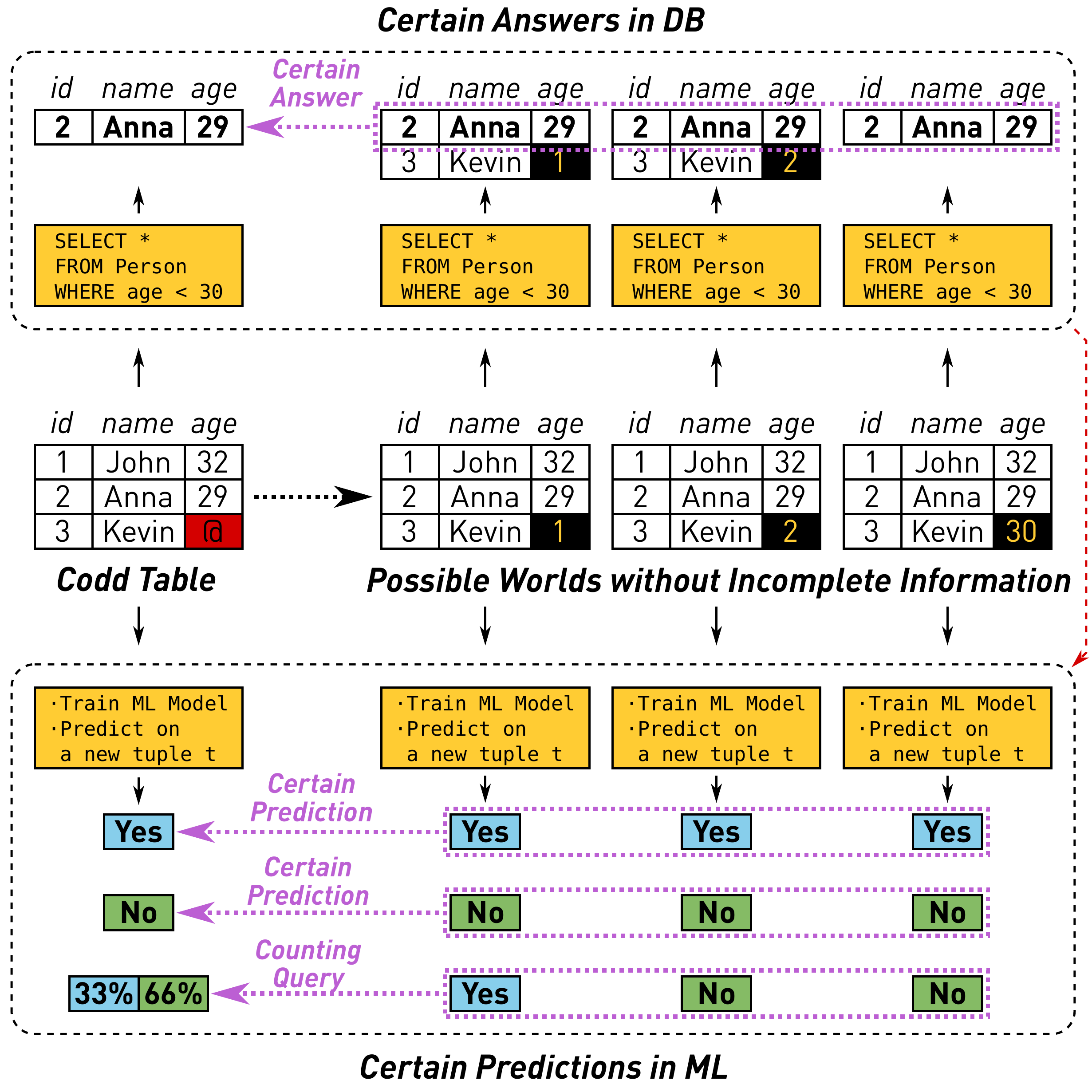}
\caption{An illustration of the
relationship between \textit{certain answers}
and \textit{certain predictions}.}
\label{fig:sureanswer}
\end{figure}

Figure~\ref{fig:sureanswer} illustrates 
one dataset with
incomplete information. In this example,
we have the incomplete dataset $D$ with
one missing cell (we will
focus on cases in which there are 
\textit{many} cells with incomplete information) --- the 
age of Kevin is not 
known and therefore is set as \texttt{NULL} (@). 
Given an ML training algorithm $\mathcal{A}$,
we can train an ML model over $D$, $\mathcal{A}_D$, and given a clean test
example $t$, we can get the prediction of 
this ML model $\mathcal{A}_D(t)$. The focus 
of this paper is to understand \textit{how much impact
the incomplete information (@) has
on the prediction $\mathcal{A}_D(t)$}. This 
question is not only of theoretical interest
but can also have interesting practical implications ---
for example, if we know that,
for a large enough number of samples of
$t$, the incomplete information
(@) does not have an impact on $\mathcal{A}_D(t)$
at all, spending the effort of cleaning or acquiring
this specific piece of missing information will
not change the quality of downstream ML models.

\vspace{0.5em}
\noindent
{\bf \textit{Relational Queries over 
Incomplete Information}.}
This paper is 
inspired by
the \textit{algorithmic}
and \textit{theoretical} foundations of 
running {\em relational queries
over incomplete information}~\cite{Alicebook}.
In traditional database theory,
there are multiple ways of representing
incomplete information, starting from the 
Codd table, or the conditional table (c-table),
all the way to the recently studied probabilistic 
conditional table (pc-table)~\cite{PDBBOOK}. Over
each of these representations of
incomplete information,
one can define the corresponding 
semantics of a relational query.
In this paper, we focus on the 
weak representation system built upon
the Codd table, as illustrated in 
Figure~\ref{fig:sureanswer}.
Given a Codd table
$T$ with constants and $n$ variables
over domain $\mathcal{D}_v$
(each variable only appears once and
represents the incomplete information
at the corresponding cell),
it represents $|\mathcal{D}_v|^n$
many \textit{possible worlds} $rep(T)$,
and a query $\texttt{Q}$
over $T$ can be defined as
returning the \textit{certain answers}
that \textit{always} appear in the answer of 
$\texttt{Q}$ over each possible 
world:
\[
sure(\texttt{Q}, T) = \cap\{\texttt{Q}(I) | I \in rep(T)\}.
\]
Another line of work with similar
spirit is \textit{consistent query
answering}, which was first introduced
in the seminal work by Arenas, Bertossi, and Chomicki~\cite{DBLP:conf/pods/ArenasBC99}.
Specifically, given an inconsistent 
database instance $D$, it defines a set of repairs $\mathcal{R
}_D$, each of which is a consistent database instance.
Given a query $\texttt{Q}$, a tuple $t$ is a {\em consistent answer}
to $\texttt{Q}$ if and only if $t$ appears in \textit{all 
answers} of $\texttt{Q}$ evaluated on every consistent 
instance $D' \in \mathcal{R}_D$. 

Both lines of work lead to a similar way of
thinking in an effort to reason about
data processing over incomplete information, i.e., 
{\em to reason about 
certain/consistent answers over all possible 
instantiations of incompleteness and uncertainty.}

\vspace{0.5em}
\noindent
{\bf \emph{Learning Over
Incomplete Information: Certain
Predictions (CP).}}
The traditional database view provides us a powerful
tool to reason about the impact of data incompleteness 
on downstream operations.
In this paper, we take a natural step and 
extend this to machine learning (ML) --- given a Codd table $T$,
its $|\mathcal{D}_v|^n$ many
possible worlds $rep(T)$, and
an ML classifier $\mathcal{A}$, 
one could train one 
ML model $\mathcal{A}_{I}$ for each
possible world $I \in rep(T)$. 
Given a test example
$t$, we say that $t$ can be {\em certainly predicted (CP'ed)} if 
$\forall I \in rep(T)$, $\mathcal{A}_{I}(t)$ always 
yields the same class label,
as illustrated in Figure~\ref{fig:sureanswer}.
This notion of certain prediction (CP) offers a canonical view of the impact from training classifiers on top of incomplete data.
Specificlly, we consider the following two 
CP queries:

\begin{itemize}
    \item[\textbf{(Q1)}] \textbf{Checking Query} --- Given a test data example, determine whether it can be CP'ed or not;
    \item[\textbf{(Q2)}] \textbf{Counting Query} --- Given a test data example that \emph{cannot} be CP'ed, for each possible prediction, compute the number of classifiers that \textit{support} this prediction.
\end{itemize}

When no assumptions are made about the classifier, Q1 and Q2 are,
not surprisingly, hard.
In this paper, we focus on (1)
developing efficient 
solutions to both Q1 and Q2
for a specific family of classifiers,
while (2) in the meantime,
trying to understand the
empirical implication and
application of CP to
the emerging research topic of
\textit{data cleaning for machine learning}.

\vspace{0.5em}
\noindent
{\bf \emph{Efficient CP
Algorithm for Nearest Neighbor 
Classifiers.}} We first study 
efficient algorithms to answer
both CP queries for 
K-nearest neighbor (KNN) classifier, one of the most popular classifiers used in practice.
Surprisingly, we show that,
\textit{both CP queries can 
be answered in polynomial time,
in spite of there being exponentially
many possible worlds}!

Moreover, these algorithms can be made 
very efficient. For example,
given a Codd table with $N$ rows and at most $M$ possible versions for rows with missing values, 
we show that answering 
both queries only take 
$\mathcal{O} \left(N \cdot M \cdot (\log(N\cdot M) +  K \cdot \log N ) \right)$. For Q1 in the binary 
classification case, we can
even do 
$\mathcal{O} (N \cdot M)$!
This makes it possible to efficiently
answer both queries for the KNN classifier,
a result that is both 
\textit{surprising} (at least to us),
\textit{new}, and
\textit{technically non-trivial}.

{\em \underline{Discussion}: Relationship with 
answering KNN queries 
over probabilistic databases.} As we will see
later, our result can be used to
evaluate a KNN classifier over a 
tuple-independent database,
in its standard semantics~\cite{10.1145/2213556.2213588,10.1145/2955098,10.5555/1783823.1783863}.
Thus we hope to draw the reader's attention 
to an interesting line of work of
evaluating KNN
\textit{queries} over a probabilistic database
in which the user wants the system to return 
the probability of a given (in our setting,
training) tuple that is 
in the top-K list of a query. 
Despite the similarity of the naming
and the underlying data model,
we focus on a different problem in this paper
as we care about the result of a KNN \textit{\underline{classifier}} instead of a KNN \textit{\underline{query}}.
Our algorithm is very different and heavily  relies on the structure of the classifier.

\vspace{0.5em}
\noindent
{\bf \emph{Applications to Data Cleaning for Machine Learning.}} The above 
result is not only of theoretical 
interest, but also has an interesting
empirical implication --- \textit{intuitively,
the notion of CP provides us a way 
to measure the relative importance 
of different variables in the Codd table
to the downstream classification accuracy}.
\update{Inspired by this intuition,
we study the efficacy of CP in the imporant application of ``data cleaning for machine learning (DC for ML)''~\cite{krishnan2017boostclean,DBLP:journals/pvldb/KrishnanWWFG16}.
Based on the CP framework,
we develop a novel algorithm CPClean that
prioritizes manual cleaning 
efforts given a dirty dataset.}

Data cleaning (DC) is often an important prerequisite step in the entire pipeline of an ML application.
Unfortunately, most existing work
considers DC as a standalone exercise without considering its impact on downstream ML applications
(exceptions include exciting seminal work such as
ActiveClean~\cite{DBLP:journals/pvldb/KrishnanWWFG16} and BoostClean~\cite{krishnan2017boostclean}). Studies have shown that such \emph{oblivious} data cleaning may not necessarily improve downstream ML models' performance~\cite{li2019cleanml}; worse yet, it can sometimes even degrade ML models' performance due to Simpson's paradox~\cite{DBLP:journals/pvldb/KrishnanWWFG16}. 
We propose a novel ``DC for ML'' framework built on top of certain predictions.
In the following discussion, we assume a standard setting for building ML models, where we are given a training set $D_{\text{train}}$ and a validation set $D_{\text{val}}$ that are drawn independently from the same underlying data distribution.
We assume that $D_{\text{train}}$ may contain missing information  whereas $D_{\text{val}}$ is complete.

The intuition of our framework is
as follows.
When the validation set is sufficiently large, if Q1 returns \verb|true| for every data example $t$ in $D_{\text{val}}$, then with high probability cleaning $D_{\text{train}}$ will not have impact on the model accuracy.
In this case we can immediately finish without any human cleaning effort.
Otherwise, some data examples cannot be CP'ed, and our goal is then to 
clean the data such that all these
examples can be CP'ed.
\emph{Why is this sufficient}?
The key observation is that, as long as a tuple $t$ can be CP'ed, 
the prediction will remain
the same regardless of further cleaning efforts.
That is, even if we clean the whole $D_{\text{train}}$, the prediction for $t$ (made by the classifier using the \emph{clean} $D_{\text{train}}$) will remain the same, simply because the final clean version is one of the possible worlds of $D_{\text{train}}$ that has been included in the definition of CP!

To \emph{minimize} the number of tuples in $D_{\text{train}}$ being cleaned until all data examples in $D_{\text{val}}$ are CP'ed, we further propose a novel optimization algorithm based on
the principle of \emph{sequential information maximization}~\cite{Chen2015SequentialIM}, exploiting the counts in Q2 for each example in $D_{\text{val}}$ that cannot be certainly predicted.
The optimization algorithm is \emph{iterative}: Each time we pick the next example in $D_{\text{train}}$ (to be cleaned) based on its potential impact on the ``degree of certainty'' of $D_{\text{train}}$ after cleaning (see Section~\ref{sec:seq-info-max} for more details).
%




\paragraph*{Summary of Contributions}
In summary, this paper makes the following contributions:
\begin{itemize}
    \item [\textbf{(C1)}] We propose \textit{certain predictions}, as well as its two fundamental queries/primitives (checking and counting), as a tool to study the impact of incomplete data on training ML models. 
    \item [\textbf{(C2)}] We propose efficient solutions to the two fundamental CC queries for nearest neighbor classifiers, despite the hardness of these two queries in general.
    \item [\textbf{(C3)}] We propose a novel ``DC for ML'' approach, CPClean, built on top of the CP primitives that significantly outperforms existing work in terms of classification accuracy, 
    with mild manual cleaning effort.
\end{itemize}

\paragraph*{Moving Forward}
Just like the study of consistent query
answering that focuses on specific
subfamilies of queries, in this paper we have focused on a specific  
type of classifier, namely the
KNN classifier,
in the CP framework.
This allows us to design efficient 
algorithms specific to this workload.
In the future, it is interesting to extend our study to 
a more diverse range of classifiers ---
either to develop efficient exact 
algorithms or to explore efficient 
approximation algorithms. It is also interesting to 
extend our CP-based data cleaning framework
to more types of classifiers.

\paragraph*{Paper Organization} This paper is organized as follows.
We formalize the notion of certain
predictions, as well as the two primitive queries Q1 and Q2 ( Section~\ref{sec:preliminary}).
We then propose efficient algorithms
in the context of nearest neighbor classifiers (Section~\ref{sec:cc-knn}).
We follow up by proposing our novel ``DC for ML'' framework exploiting CP (Section~\ref{sec:cc:general}).
We report evaluation results in Section~\ref{sec:evaluation}, summarize related work in Section~\ref{sec:related-work}, and conclude the paper in Section~\ref{sec:conclusion}.

\section{Certain Prediction (CP)}
\label{sec:preliminary}

In this section, we describe the 
\textit{certain prediction} (CP) framework, which 
is a natural extension of the notion 
of \textit{certain answer} for 
query processing over Codd tables~\cite{Alicebook} 
to machine learning.
We first describe our data model and then
introduce two CP queries.

\paragraph*{Data Model} 

We focus on standard supervised ML settings: 
\begin{enumerate}
\item Feature Space $\mathcal{X}$: without 
loss of generality, we assume that every
data example is drawn from a domain $\mathcal{X} = \mathbb{D}^d$, i.e., a $d$ dimensional
space of data type $\mathbb{D}$.
\item Label Space $\mathcal{Y}$: we assume
that each data example can be classified into
one of the labels in $\mathcal{Y}$.
\item Training Set $D_{train} \subseteq 
\mathcal{X} \times \mathcal{Y}$ is drawn
from an \textit{unknown} distribution $\mathcal{P}_{\mathcal{X}, \mathcal{Y}}$.
\item Test Set $D_{test} \subseteq 
\mathcal{X}$
(Validation Set $D_{val}$)
is drawn from the marginal distribution $\mathcal{P}_{\mathcal{X}}$ of the joint distribution $\mathcal{P}_{\mathcal{X}, \mathcal{Y}}$.
\item Training Algorithm $\mathcal{A}$:
A training algorithm $\mathcal{A}$ is
a functional that maps a given training
set $D_{train}$ to a function $\mathcal{A}_{D_{train}}: \mathcal{X} \mapsto \mathcal{Y}$. Given a test example 
$t \in D_{test}$, $\mathcal{A}_{D_{train}}(t)$
returns the prediction of the trained
classifier on the test example $t$.
\end{enumerate}

\begin{figure}
    \centering
    \includegraphics[width=0.5\textwidth]{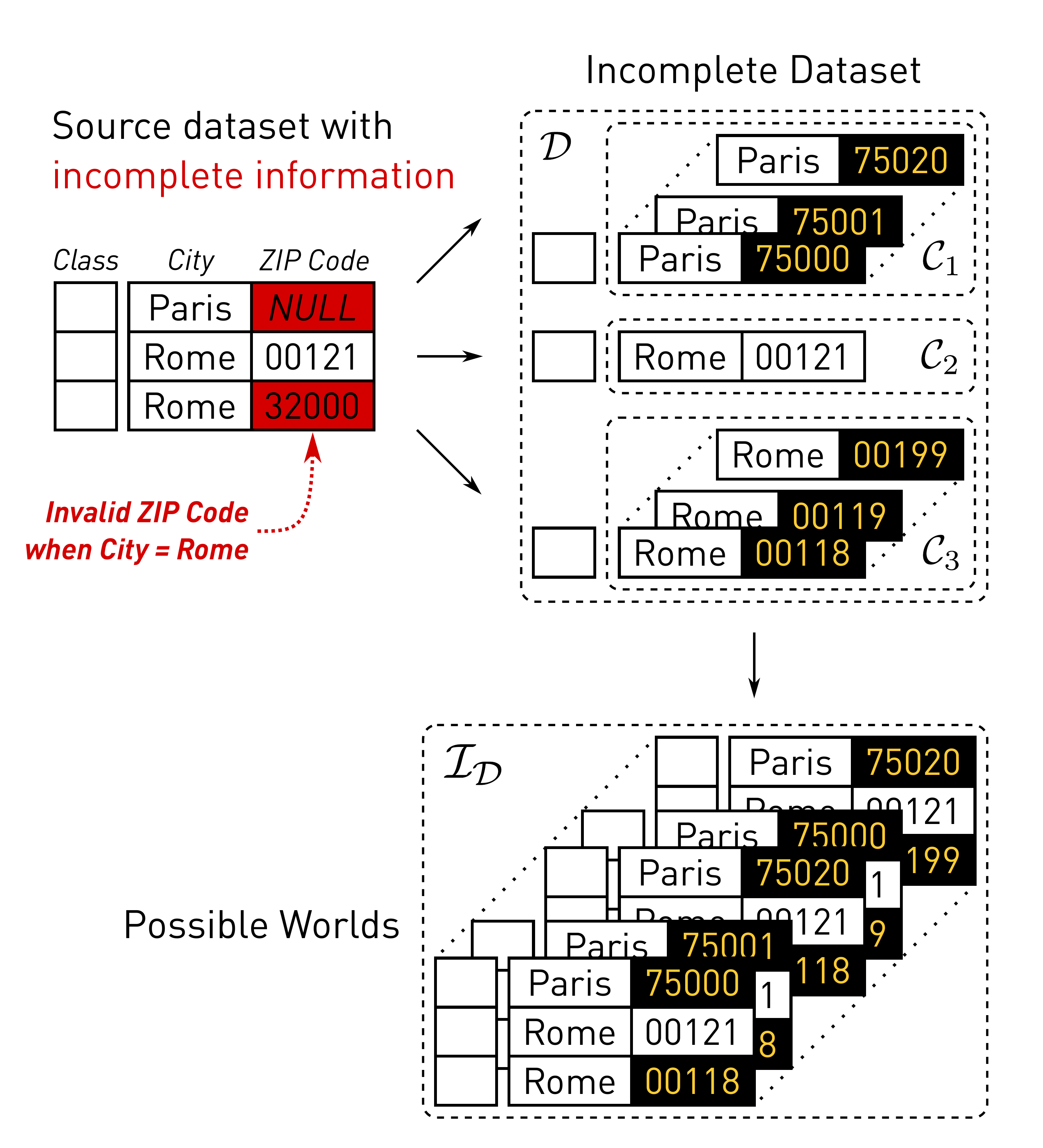}
    \caption{Example of a dataset with incomplete information, its representation as an \emph{incomplete dataset}, and the induced set of \emph{possible worlds}.}
    \label{fig:possible-worlds}
\end{figure}

\noindent
{\bf \emph{Incomplete Information in the Training
Set}} In this paper, we focus on the case in which
there is incomplete information in the
training set. We define an \textit{incomplete
training set} as follows.

Our definition of an incomplete
training set is very similar to a
block tuple-independent 
probabilistic database~\cite{PDBBOOK}.
However,
we do assume that there is no
uncertainty on the label and we do not have 
access to the probability distribution of each tuple.
%
\begin{definition}[Incomplete Dataset] \label{def:incomplete-dataset}
An \emph{incomplete dataset}
\[
\mathcal{D} = \{(\mathcal{C}_i, y_i) : i=1,...,N\}
\]
is a finite set of $N$ 
pairs where each $\mathcal{C}_i = \{ x_{i,1}, x_{i,2}, ... \} \subset \mathcal{X}$ is a finite number 
of possible feature vectors of the $i$-th data example and each $y_i \in \mathcal{Y}$ is its corresponding class label.
\end{definition}

According to the semantics of $\mathcal{D}$, the $i$-th data example can take any of the values from its corresponding 
\textit{candidate set} $\mathcal{C}_i$. The space of all possible ways to assign values to all data points in $\mathcal{D}$ is captured by the notion of possible worlds.
Similar to a block tuple-independent 
probabilistic database,
an incomplete dataset can define a
set of \textit{possible worlds}, each of which is a 
dataset without incomplete information.

\begin{definition}[Possible Worlds] \label{def:clean-world}
Let $\mathcal{D} = \{(\mathcal{C}_i, y_i) : i=1,...,N\}$ be an incomplete dataset. We define the set of possible worlds $\mathcal{I}_\mathcal{D}$, given the incomplete dataset $\mathcal{D}$, as
$$
\mathcal{I}_{\mathcal{D}} \ = \
 \left\{
 D = \{(x_i', y_i')\} :
 |D| = |\mathcal{D}| \wedge 
 \forall i.~x_i' \in \mathcal{C}_i \wedge y_i' = y_i 
 \right\}.
$$
\end{definition}

In other words, a 
\emph{possible world} represents 
one complete dataset $D$ that is generated from $\mathcal{D}$ by replacing every
candidate set $\mathcal{C}_i$ with one of its candidates $x_j \in \mathcal{C}_i$. The set of all distinct datasets that we can generate in this way is referred to as the \emph{set of possible worlds}. If we assume that $\mathcal{D}$ has $N$ data points and the size of each $\mathcal{C}_i$ is bounded by $M$, we can see $|\mathcal{I}_{\mathcal{D}}| = \mathcal{O} (M^N)$. 

Figure~\ref{fig:possible-worlds} provides an example of these concepts. As we can see, our definition of incomplete dataset can represent both possible values for  missing cells and possible repairs for cells that are considered to be potentially incorrect.

\vspace{0.5em}
\noindent
{\bf \emph{Connections to Data Cleaning.}} 
In this paper, we 
use data cleaning as one
application to illustrate the 
practical implication of 
the CP framework. In this setting,
each possible world
can be thought of as one 
possible \emph{data repair} 
of the dirty/incomplete data.
These repairs can be generated in an arbitrary way, possibly depending on the entire dataset~\cite{rekatsinas2017holoclean}, or even some external domain knowledge~\cite{ChuKATARA}. Attribute-level data repairs could also be generated independently and merged together with Cartesian products. 


We will further  
apply the assumption that any given incomplete dataset $\mathcal{D}$ is \emph{valid}. That is, for every data point $i$, we assume that there exists a \emph{true value} $x_i^*$ that is unknown to us, but is nevertheless included in the candidate set $\mathcal{C}_i$. 
This is a commonly used assumption in data cleaning~\cite{DBLP:books/acm/IlyasC19}, where automatic cleaning algorithms are used to generate a set of candidate repairs, and humans are then asked to pick one from the given set.
We call 
$D_\mathcal{D}^*$ the \textit{true 
possible world}, which contains 
the true value for each tuple.
When $\mathcal{D}$ is clear from the context, 
we will also write $D^*$.

\subsection{Certain Prediction (CP)}

When we train an ML model over
an incomplete dataset, we can define
its semantics in a way that is very 
similar to how people define 
the semantics for data processing over
probabilistic databases ---
we denote $\mathcal{A}_{D_i}$
as the classifier that was trained on the possible world $D_i \in \mathcal{I}_\mathcal{D}$. 
Given a test data point $t \in \mathcal{X}$, 
we say that it can be \textit{certainly
predicted} (CP'ed) if 
all classifiers trained on 
all different possible worlds
agree on their predictions:

\begin{definition}[Certain Prediction (CP)] \label{def:cc}
Given an incomplete dataset $\mathcal{D}$ with its set of possible worlds $\mathcal{I}_\mathcal{D}$ and a data point $t \in \mathcal{X}$, we say that a label $y \in \mathcal{Y}$ can be \emph{certainly predicted} with respect to a learning algorithm $\mathcal{A}$ if and only if
$$
\forall D_i \in \mathcal{I}_\mathcal{D}, \mathcal{A}_{D_i}(t) = y.
$$
\end{definition}

\noindent
{\bf \emph{Connections to Databases.}} 
The intuition behind this definition is rather 
natural from the perspective of database theory. In the context of Codd table,
each \texttt{NULL} variable 
can take values in its domain, which
in turn defines exponentially many
possible worlds~\cite{Alicebook}. Checking 
whether a tuple is in the answer of
some query \texttt{Q} is to check whether such a 
tuple is in the result of 
each possible
world.

\begin{figure}[t!]
    \centering
    \includegraphics[width=0.6\textwidth]{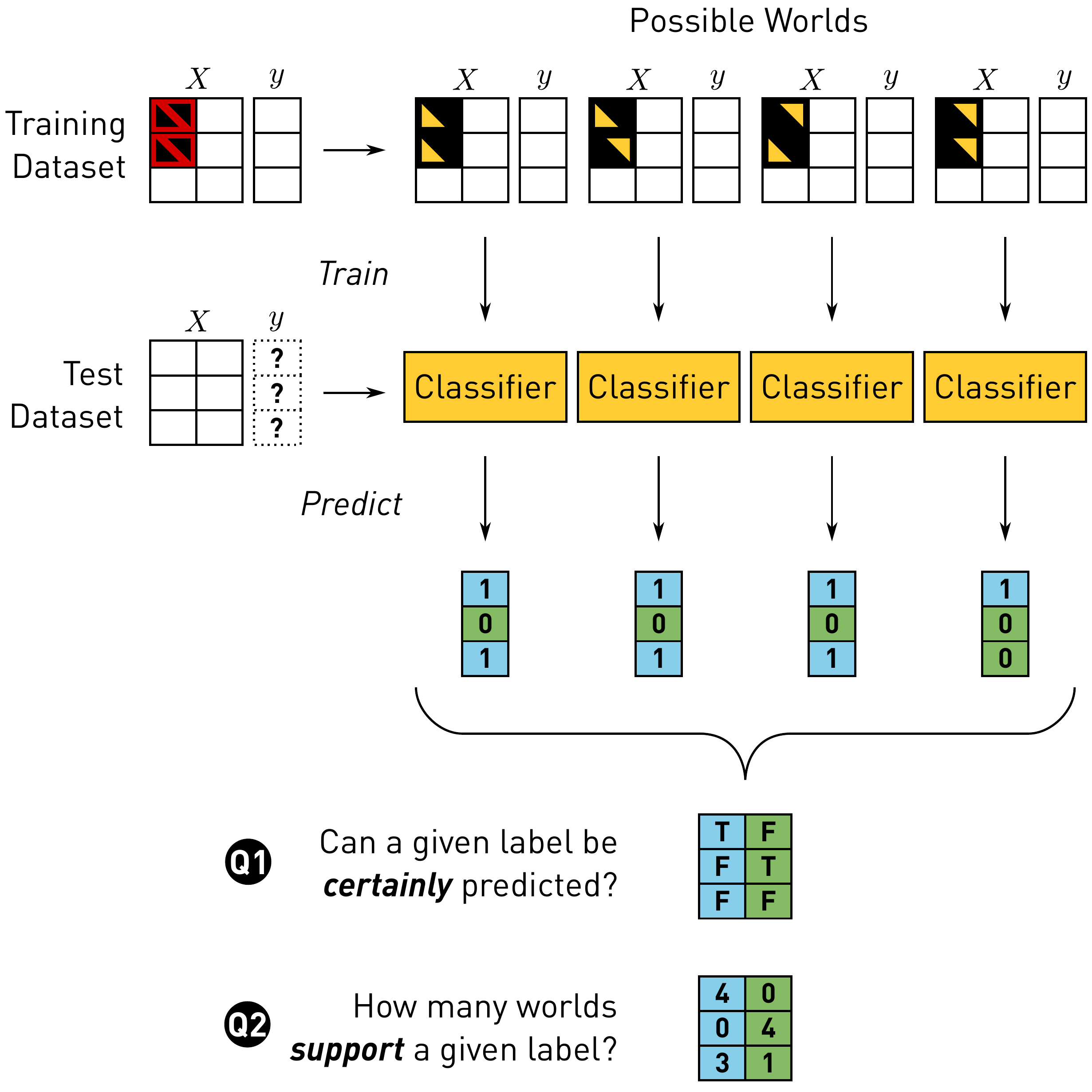}
    \caption{Illustration of certain prediction, and two queries: checking query (Q1) and counting query (Q2).}
    \label{fig:cc}
\end{figure}

\paragraph*{Two Primitive CP Queries}

Given the notion of \textit{certain prediction},
there are two natural 
queries that we can ask.
The query $Q1$ represents a \emph{decision problem} that checks if a given label can be predicted in \textit{all} possible worlds. The query $Q2$ is an extension of that and represents a \emph{counting problem} that returns the number of possible worlds that support each prediction outcome.
Figure~\ref{fig:cc} illustrates both
queries and we formally define
them as follows.

\begin{definition}[Q1: Checking]
Given a data point $t \in \mathcal{X}$, an incomplete dataset $\mathcal{D}$ and a class label $y \in  \mathcal{Y}$, we define a query that checks \emph{if all possible world} permits $y$ to be predicted:
$$
Q1(\mathcal{D}, t, y) := \begin{cases}
	\mathtt{true}, & \mathrm{if} \ \forall D_i \in \mathcal{I}_\mathcal{D}, \mathcal{A}_{D_i}(t) = y; \\
	\mathtt{false}, & \mathrm{otherwise}.
\end{cases}
$$
\end{definition}

\begin{definition}[Q2: Counting]
Given a data point $t \in \mathcal{X}$, an incomplete dataset $\mathcal{D}$ and a class label $y \in  \mathcal{Y}$, we define a query that returns the \emph{number of possible worlds} that permit $y$ to be predicted:
$$
Q2(\mathcal{D}, t, y) := | \{ D_i \in \mathcal{I}_\mathcal{D} : \mathcal{A}_{D_i}(T) = y \}|.
$$
\end{definition}

\noindent
\textbf{\textit{Computational Challenge.}} If we do not make any assumption about the learning algorithm $\mathcal{A}$, we have no way of determining the predicted label $y=\mathcal{A}_{D_i}(t)$ except for running the algorithm on the training dataset. Therefore, for a general classifier treated as a black box, answering both $Q1$ and $Q2$ requires us to apply a brute-force 
approach that iterates over each $D_i \in \mathcal{I}_\mathcal{D}$, produces $\mathcal{A}_{D_i}$, and predicts the label. 
Given an incomplete dataset with
$N$ data examples each of which has
$M$ clean candidates, the computational cost of this naive algorithm for both queries would thus be $\mathcal{O} (M^N)$.

This is not surprising. 
However, 
as we will see later in this paper,
for certain types of classifiers, such as K-Nearest Neighbor classifiers,
we are able to design efficient
algorithms for both queries.

\vspace{0.5em}
\noindent
{\bf \emph{Connections to Probabilistic 
Databases.}}
Our definition of certain prediction 
has strong connection to the theory of probabilistic 
database~\cite{PDBBOOK} --- in fact, Q2 can be seen as a natural 
definition of evaluating
an ML classifier over a block
tuple-independent probabilistic database 
with uniform prior.


Nevertheless, unlike traditional relational queries over a probabilistic database, our ``query''
is an ML model that has very different structure.
As a result, despite the fact that
we are inspired by many seminal works
in probabilistic database~\cite{10.1145/2213556.2213588,10.1145/2955098,10.5555/1783823.1783863}, 
they are not applicable to our settings and 
we need to develop new techniques.

\vspace{0.5em}
\noindent
{\bf \emph{Connections to Data Cleaning.}}
It is easy to see that, if $Q1$
returns \texttt{true} on a test example
$t$, obtaining 
more information (by cleaning) for
the original training set 
will not change the prediction 
on $t$ at all! This is because 
the true possible world $D^*$
is one of the possible worlds in
$\mathcal{I}_\mathcal{D}$.
Given a large enough test set,
if $Q_1$ returns \texttt{true} for
all test examples, cleaning the training set
in this case might not
improve the quality of ML models
at all!

Of course, in practice, it is 
unlikely that all test examples 
can be CP'ed. In this more realistic
case, $Q_2$ provides a ``softer''
way than $Q_1$ to measure the 
\textit{degree of certainty/impact}.
As we will see later, we can use this
as a principled proxy of the impact
of data cleaning on downstream ML
models, and design efficient algorithms to prioritize which uncertain cell 
to clean in the training set.

\begin{figure}[t!]
\centering
\small
\begin{tabular}{cccccc}
\hline
$K$ & $|\mathcal{Y}|$ & Query & Alg. & Complexity
in $O(-)$&  Section \\
\hline 
1 & 2 & Q1/Q2 & SS & $NM \log NM$ &  3.1.2 \\
$K$ & 2 & Q1 & MM & $NM$ &  3.2 \\
$K$ & $|\mathcal{Y}|$ & Q1/Q2 & SS & $NM (\log(NM) +  K\update{^2} \log N$) &  3.1.3 \\
\hline
\end{tabular}
\vspace{-0.5em}
\caption{Summary of results ($K$ and $|\mathcal{Y}|$ are constants).
}
\label{fig:summary_of_result}
\end{figure}

\section{Efficient Solutions for CP Queries}\label{sec:cc-knn}



Given our definition of certain prediction,
not surprisingly, both queries are  
hard if we do not assume any structure
of the classifier. In this section, we 
focus on a specific classifier that is
popularly used in practice, namely the 
$K$-Nearest Neighbor (KNN) classifier.
As we will see, for a KNN classifier,
we are able to answer both CP queries
in \textit{polynomial} time, even though we are
reasoning over \textit{exponentially} many
possible worlds!

\vspace{0.5em}
\noindent
{\bf \emph{$K$-Nearest Neighbor Classifiers.}} 
A textbook KNN classifier works in the following way, as illustrated in 
Figure~\ref{fig:knn-similarity}(a): Given a training set 
$D = \{(x_i, y_i)\}$ and a test example 
$t$, we first calculate the similarity
between $t$ and each $x_i$: $s_i =  \kappa(x_i, t)$.
This similarity can be calculated 
using different kernel functions $\kappa$
such as linear kernel, RBF kernel, etc.
Given all these similarity scores
$\{s_i\}$, we pick the top $K$ training
examples with the largest similarity score:
$x_{\sigma_1}, ..., x_{\sigma_K}$
along with corresponding labels
$\{y_{\sigma_i}\}_{i\in[K]}$.
We then take the majority 
label among $\{y_{\sigma_i}\}_{i\in[K]}$ and return 
it as the prediction for the test example
$t$.


\vspace{0.5em}
\noindent
{\bf \emph{Summary of Results.}} 
In this paper, we focus on designing
efficient algorithms to support a KNN
classifier for both CP queries.
In general, all these results 
are based on two algorithms,
namely SS (SortScan) and MM (MinMax).
SS is a generic algorithm that
can be used to answer 
both queries, while MM can only be used
to answer $Q1$. However, on the other
hand, MM permits lower complexity 
than SS when applicable.
Figure~\ref{fig:summary_of_result} summarizes the result.

\vspace{0.5em}
\noindent
{\bf \emph{Structure of This Section.}} 
In Section~\ref{sec:ss}
we will focus on the SS 
algorithm as it is more 
generic. We will
explain a simplified version
of the SS algorithm for 
the special case ($K=1$, $|\mathcal{Y}|=2$)
in greater details as it conveys
the intuition behind this algorithm.
We will follow by describing the
SS algorithm in its general form.
We will summarize the
MM algorithm
in Section~\ref{sec:mm}, which
can be significantly more efficient than SS
in some cases, but leave
the full details to the appendix.

\begin{figure}[t!]
    \centering
    \includegraphics[width=0.5\textwidth]{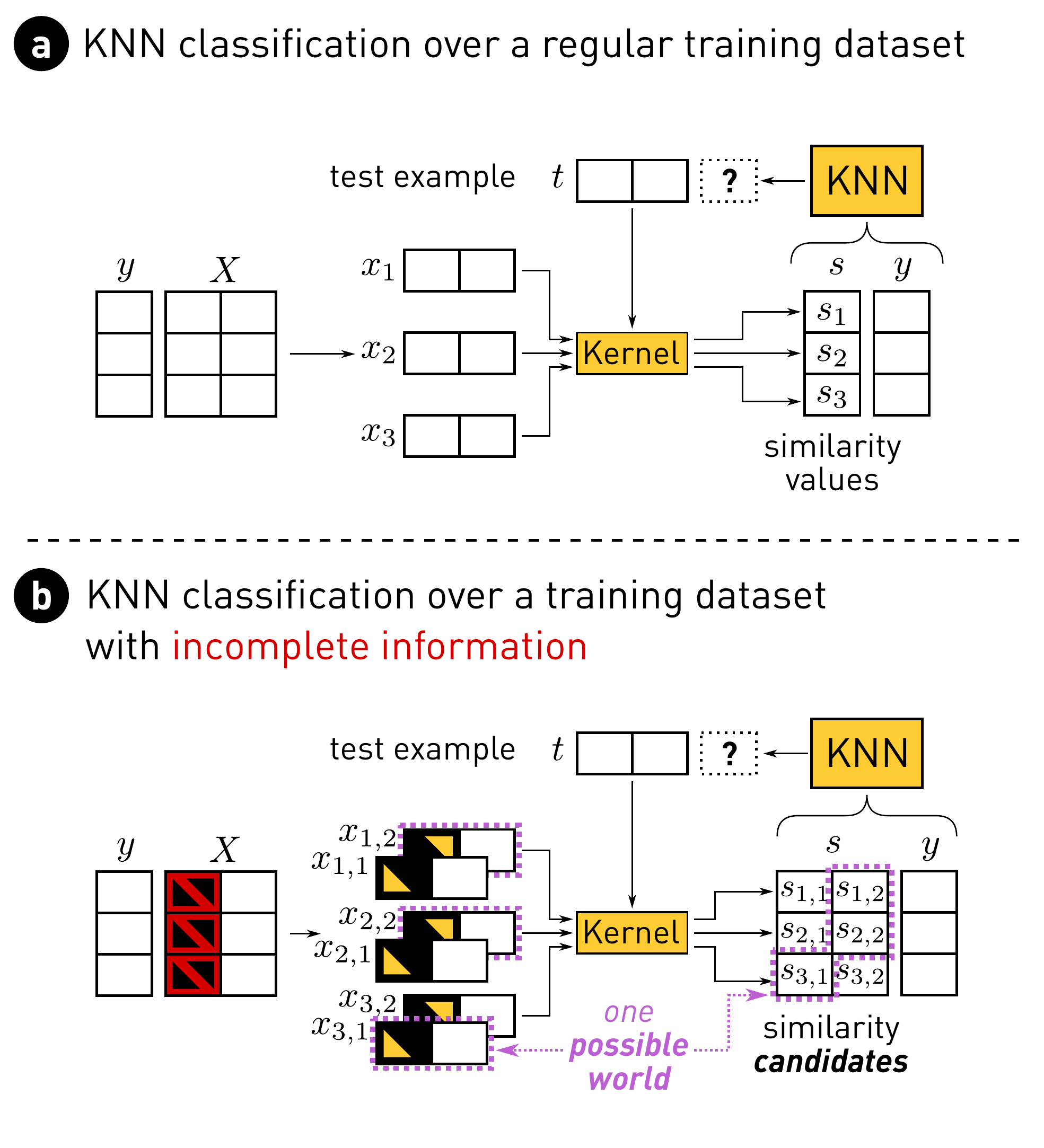}
    \caption{Illustration of KNN classifiers.}
    \label{fig:knn-similarity}
\end{figure}

\subsection{SS Algorithm}\label{sec:ss}

We now describe the SS algorithm.
The idea behind SS is that we can calculate
the similarity between all
candidates $\cup_i \mathcal{C}_i$ in an incomplete 
dataset and a test example $t$.
Without loss of generaility,
assume that $|\mathcal{C}_i| = M$, this leads to $N \times M$
similarity scores $s_{i, j}$.
We can then sort and scan these
similarity scores.

The core of the SS algorithm
is a dynamic programming 
procedure. 
We will first describe 
a set of basic building blocks of this problem, and
then introduce
a simplified version of SS
for the special case of
$K=1$ and $|\mathcal{Y}|=2$, to explain the
intuition behind SS.
We follow this by the general version of the SS algorithm.

\subsubsection{Two Building Blocks} \label{sec:alg:ss:two-building-blocks}

In our problem, we can construct two 
building blocks efficiently. We start by
articulating the settings precisely.
In the next section, we will use these
two building blocks for our SS algorithm.


\paragraph*{Setup} We are given 
an incomplete dataset $\mathcal{D}=\{(\mathcal{C}_i, y_i)\}$.
Without loss of generality, we 
assume that each $\mathcal{C}_i$
only contains $M$ elements, i.e.,
$|\mathcal{C}_i| = M$.
We call $\mathcal{C}_i = \{x_{i,j}\}_{j \in [M]}$ the 
$i^{th}$ incomplete data example, and $x_{i, j}$
the $j^{th}$ candidate value 
for the $i^{th}$ incomplete data example.
This defines $M^N$ many possible worlds:
\[
\mathcal{I}_{\mathcal{D}} = \{D=\{(x_i^D, y_i^D)\}: 
 |D| = |\mathcal{D}| \wedge 
 y_i^D = y_i \wedge x_i^D \in \mathcal{C}_i
 \}.
\]
We use $x_{i, j_{i, D}}$ to denote the candidate value for the $i^{th}$ data point in $D$.
Given a test example $t$, we can calculate the
similarity between each candidate value $x_{i, j}$
and $t$: $s_{i, j} = \kappa(x_{i, j}, t)$. We
call these values \textit{similarity candidates}, as shown in Figure~\ref{fig:knn-similarity} (b). We assume that
there are no ties in these similarities scores (we can always break a tie by favoring a smaller
$i$ and $j$ or a pre-defined random order).

Furthermore, given a candidate value $x_{i, j}$, we 
count, for each candidate set, how many 
candidate values are less similar to the 
test example than $x_{i, j}$. This gives
us what we call the 
{\em similarity tally} $\alpha$.
For each candidate set $\mathcal{C}_n$,
we have
\[
\alpha_{i, j}[n] = \sum\nolimits_{m=1}^M \mathbb{I}[s_{n, m} \leq s_{i, j}].
\] 

\begin{figure}
    \centering
    \includegraphics[width=0.6\textwidth]{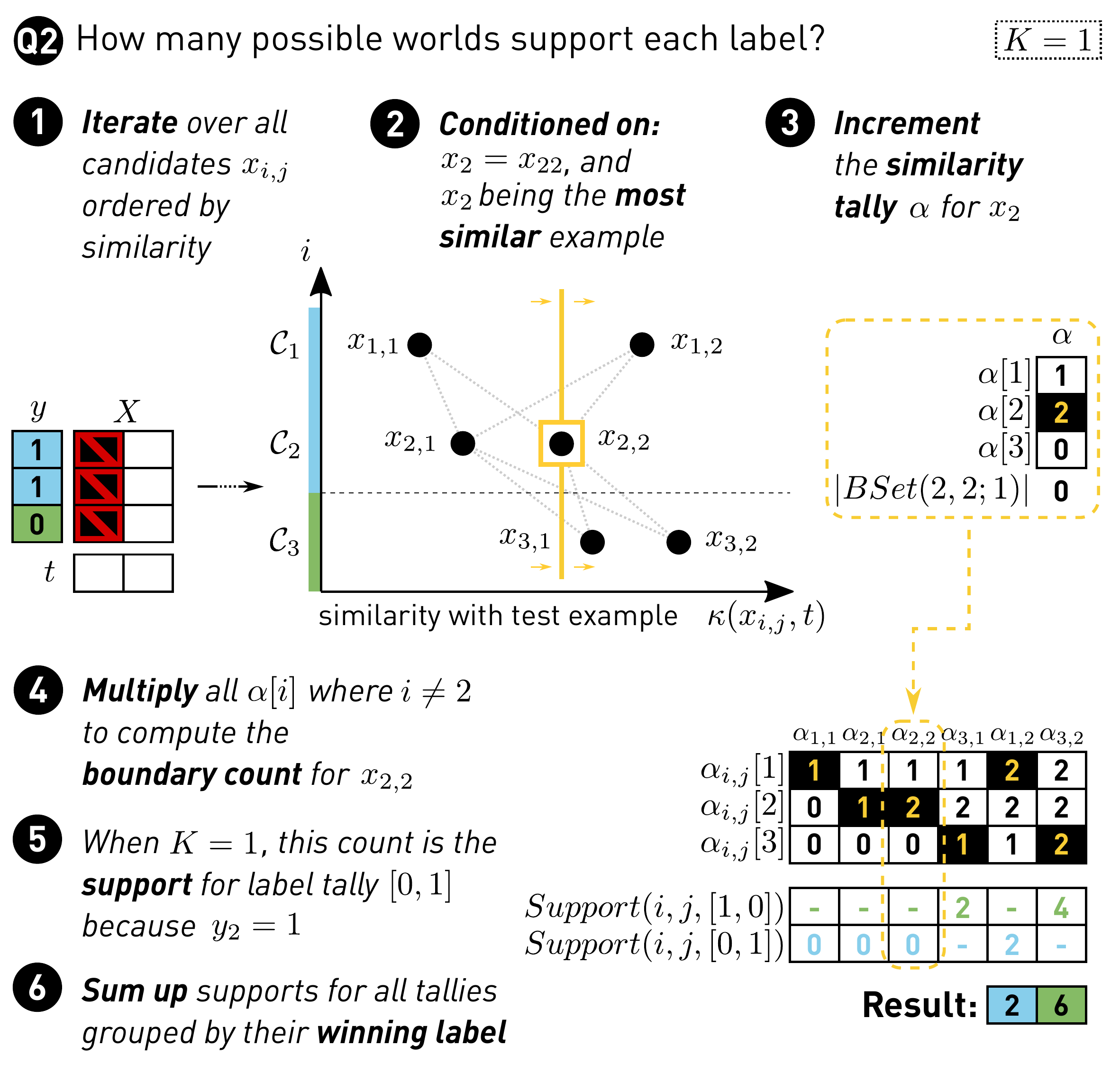}
    \caption{Illustration of SS when $K=1$ for $Q2$.}
    \label{fig:alg-sort-count-k1}
\end{figure}

\begin{example} 
In Figure~\ref{fig:alg-sort-count-k1} we can see an example of a similarity tally $\alpha_{2,2}$ with respect to the data point $x_{2,2}$. For 
$i^{th}$ incomplete
data example, it contains the number of candidate values $x_{i,j} \in \mathcal{C}_i$ that have the similarity value no greater than $s_{2,2}$. Visually, in Figure~\ref{fig:alg-sort-count-k1}, this represents all the candidates that lie left of the vertical yellow line. We can see that only one candidate from $\mathcal{C}_{1}$, two candidates from $\mathcal{C}_{2}$, and none of the candidates from $\mathcal{C}_{3}$ satisfy this property. This gives us $\alpha_{2,2}[1]=1$, $\alpha_{2,2}[2]=2$, and $\alpha_{2,2}[3]=0$.
\end{example}

\paragraph*{KNN over Possible World $D$}
Given one possible world $D$, running a KNN classifier
to get the prediction for a test example $t$ 
involves multiple stages. 
First, we obtain
{\em Top-K Set}, the set of $K$ examples in $D$
that are in the K-nearest neighbor set
\[
Top(K, D, t) \subseteq [N],
\]
which has the following property
\[
|Top(K, D, t)| = K,
\]
\begin{eqnarray*}
\forall i, i' \in [N].~~~i \in Top(K, D, t) \wedge i' \not
\in Top(K, D, t) 
\implies  s_{i, j_{i, D}} > s_{i', j_{i', D}}.
\end{eqnarray*}

Given the top-K set, we then tally the corresponding
labels by counting how many 
examples in the top-K set support a given label.
We call it the {\em label tally} $\gamma^D$:
\[
\gamma^D \in \mathbb{N}^{|\mathcal{Y}|}: \gamma_l^D = 
    \sum\nolimits_{i \in Top(K, D, t)} \mathbb{I}[l = y_i].
\]
Finally, we pick the label with the largest
count:
\[
y^*_D = \arg\max_l \gamma_l^D.
\]

\begin{example}
For $K=1$, the Top-K Set contains only one element $x_i$ which is most similar to $t$. The label tally then is a $|\mathcal{Y}|$-dimensional binary vector with all elements being equal to zero except for the element corresponding to the label $y_i$ being equal to one. Clearly, there are $|\mathcal{Y}|$ possible such label tally vectors.
\end{example}

\paragraph*{Building Block 1: Boundary Set}
The first building block answers the 
following question: {\em 
Out of all possible worlds that picked 
the value $x_{i, j}$ for $\mathcal{C}_i$,
how many of them have $x_{i,j}$ as the
least similar item in the Top-K set?}
We call all possible worlds that satisfy
this condition the {\em Boundary Set}
of $x_{i, j}$:
\begin{equation*}
BSet(i, j; K)  = 
\{D: j_{i, D} = j \wedge i \in Top(K, D, t)
\wedge 
     i \not \in Top(K-1, D, t) \}.  
\end{equation*}
We call the size of the boundary 
set the {\em Boundary Count}.

We can enumerate all $N \choose (K - 1)$ 
possible configurations of the top-(K-1)
set to compute the boundary count.
Specifically, let $\mathcal{S}(K-1, [N])$ be 
all subsets of $[N]$ with size $K-1$. We have 
\[
|BSet(i, j; K)| = \sum_{\substack{S \in \mathcal{S}(K-1, [N]) \\
                i \not \in S}}
    \left(\prod_{n \not \in S} \alpha_{i, j}[n]\right) 
    \cdot 
    \left(\prod_{n \in S} (M - \alpha_{i, j}[n])\right).
\] 

The idea behind this is the following --- we
enumerate all possible settings of the
top-(K-1) set: $\mathcal{S}(K-1, [N])$.
For each specific top-(K-1)
setting $S$, every candidate set in 
$S$ needs to pick a value that is
more similar than $x_{i, j}$, while
every candidate set not in $S$ needs
to pick a value that is less similar
than $x_{i, j}$. Since the choices of
value between different candidate sets
are independent, we can calculate 
this by multiplying different entries of 
the similarity tally vector $\alpha$.

We observe that calculating the 
boundary count for a value $x_{i, j}$
can be efficient when $K$ is small.
For example, if we use a 1-NN
classifier, the only $S$ that we 
consider is the empty set, and thus, 
the boundary count merely
equals $\prod_{n \in [N], n \ne i} \alpha_{i, j}[n]$.

\begin{example}
We can see this, in Figure~\ref{fig:alg-sort-count-k1} from Step 3 to Step 4,  where the size of the boundary set $|BSet(2, 2; 1)|$ is computed as the product over elements of $\alpha$, excluding $\alpha[2]$. Here, the boundary set for $x_{2,2}$ is actually empty. This happens because both candidates from $\mathcal{C}_3$ are more similar to $t$ than $x_{2,2}$ is, that is, $\alpha_{2,2}[3]=0$. Consequently, since every possible world must contain one element from $\mathcal{C}_3$, we can see that $x_{2,2}$ will never be in the Top-$1$, which is why its boundary set contains zero elements.

If we had tried to construct the boundary set for $x_{3,1}$, we would have seen that it contains two possible worlds. One contains $x_{2,1}$ and the other contains $x_{2,2}$, because both are less similar to $t$ than $x_{3,1}$ is, so they cannot interfere with its Top-$1$ position. On the other hand, both possible worlds have to contain $x_{1,1}$ because selecting $x_{1,2}$ would prevent $x_{3,1}$ from being the Top-$1$ example.
\end{example}

\paragraph*{Building Block 2: Label Support}
To get the prediction of a KNN classifier,
we can reason about the label tally vector 
$\gamma$, and not necessarily the specific 
configurations of the top-K set.
It answers the following question:
{\em 
Given a specific configuration of the 
label tally vector $\gamma$, how many
possible worlds in the boundary set of
$x_{i,j}$ support this $\gamma$?
}
We call this the {\em Support} of the
label tally vector $\gamma$:
\[
Support(i, j, \gamma) = |\{D: \gamma^D = \gamma \wedge 
D \in BSet(i, j; K)\}|.
\]

\begin{example}
For example, when $K=3$ and $|\mathcal{Y}|=2$, we have $4$ possible label tallies: $\gamma \in \{ [0,3], [1,2], [2,1], [3,0] \}$. Each tally defines a distinct partition of the boundary set of $x_{i,j}$ and the size of this partition is the support for that tally. Note that one of these tallies 
always has support 0,
which happens when $\gamma_l = 0$ for the label $l=y_i$, thus excluding $x_{i,j}$ from the top-$K$ set. 

For $K=1$, a label tally can only have one non-zero value that is equal to $1$ only for a single label $l$. Therefore, all the elements in the boundary set of $x_{i,j}$ can support only one label tally vector that has $\gamma_l=1$ where $l=y_{i}$. 
This label tally 
vector will always have
the support equal 
to the boundary count of $x_{i,j}$.
\end{example}



Calculating the support can be done with 
dynamic programming. First,
we can partition the whole incomplete
dataset into $|\mathcal{Y}|$ many 
subsets, each of which only contains
incomplete data points
(candidate sets) of the same label $l \in |\mathcal{Y}|$:
\[
\mathcal{D}_l = \{(\mathcal{C}_i, y_i): y_i = l \wedge 
(\mathcal{C}_i, y_i) \in \mathcal{D}\}.
\]
Clearly, if we want a possible world $D$ that
supports the label tally vector $\gamma$, 
its top-K set needs to have 
$\gamma_1$ candidate sets from $\mathcal{D}_1$,
$\gamma_2$ candidate sets from $\mathcal{D}_2$,
and so on.
{\em 
Given that $x_{i, j}$ is on the boundry, how many 
ways do we have to pick 
$\gamma_l$ many candidate sets from $\mathcal{D}_l$
in the top-K set?} We can represent this value
as $C^{i, j}_l(\gamma_l, N)$, with the following 
recursive structure:
\begin{equation*}
C_l^{i, j}(c,n) =
\begin{cases}
C^{i, j}_l(c, n-1), \quad \text{if} \ y_n \neq l, \\
\update{C^{i, j}_l(c-1, n-1), \quad \text{if} \ x_n=x_i}, \text{otherwise} \\
\alpha_{i, j}[n] \cdot C^{i, j}_l(c, n-1) + (M-\alpha_{i, j}[n]) \cdot C^{i, j}_l(c-1,n-1).
\end{cases}    
\end{equation*}
This recursion defines a process in which one
scans all candidate sets from $(\mathcal{C}_1, y_1)$ 
to $(\mathcal{C}_N, y_N)$.
At candidate set $(\mathcal{C}_n, y_n)$:
\begin{enumerate}
\item If $y_n$ is not equal to our target label $l$,
the candidate set $(\mathcal{C}_n, y_n)$ will 
not have any impact on the count.
\item If $x_n$ happens to be $x_i$, this will 
not have any impact on the count as 
$x_i$ is always in the top-K set, by definition. \update{However, this means that we have to decrement the number of available slots $c$.}
\item Otherwise, we have two choices to make:
\begin{enumerate}
\item Put $(\mathcal{C}_n, y_n)$ into the
top-K set, and there are $(M - \alpha_{i, j}[n])$
many possible candidates to choose from.
\item Do not put $(\mathcal{C}_n, y_n)$ into the
top-K set, and there are $\alpha_{i, j}[n]$
many possible candidates to choose from.
\end{enumerate}
\end{enumerate}

It is clear that this recursion can be 
computed as a dynamic program in $\mathcal{O}(N \cdot M)$ time. This DP is defined for $c \in \{0...K\}$ which is the exact number of candidates we want to have in the top-$K$, and $n \in \{1...N\}$ which defines the subset of examples $x_i : i \in \{1...N\}$ we are considering. The boundary conditions of this DP are $C_l^{i, j}(-1,n)=0$ and $C_l^{i, j}(c,\update{0})=1$.

Given the result of this dynamic programming 
algorithm for different values of $l$, we 
can calculate the support of label tally
$\gamma$:
\[
Support(i, j, \gamma) = \prod\nolimits_{l \in \mathcal{Y}}
C_{l}^{i, j}(\gamma_l, N),
\]
which can be computed in $\mathcal{O}(NM|\mathcal{Y}|)$.

\begin{example}
If we assume the situation shown in Figure~\ref{fig:alg-sort-count-k1}, we can try for example to compute the value of $Support(3,1,\gamma)$ where $\gamma=[1,0]$. 
We would have $C_0^{3,1}(1,N)=1$ because $x_{3}$ (the subset of $D$ with label 0) must be in the top-$K$, which happens only when $x_3=x_{3,1}$. On the other hand we would have $C_1^{3,1}(0,N)=2$ because both $x_{1}$ and $x_{2}$ (the subset of $D$ with label 1) must be out of the top-$K$, which happens when $x_1=x_{1,1}$ while $x_2$ can be either  equal to $x_{2,1}$ or $x_{2,2}$. Their mutual product is equal to $2$, which we can see below the tally column under $x_{3,1}$.
\end{example}

\subsubsection{$K=1$, $|\mathcal{Y}|=2$}

Given the above two building blocks, 
it is easy to develop an algorithm 
for the case $K=1$ and $|\mathcal{Y}|=2$.
In SS, we use the result of $Q2$ 
to answer both $Q1$ and $Q2$. Later we will
introduce the MM algorithm that is dedicated to
$Q1$ only.

We simply compute the number of possible 
worlds that support the prediction label being 1.
We do this by enumerating all possible 
candidate values $x_{i, j}$. If this 
candidate has label $y_i = 1$, we count
how many possible worlds have 
$x_{i, j}$ as the top-1 example, i.e.,
the boundry count of $x_{i, j}$.
We have
\[
Q2(\mathcal{D}, t, l) = \sum_{i \in [N]}
\sum_{j \in [M]} \mathbb{I}[y_i = l] \cdot |BSet(i, j; K=1)|, 
\]
which simplifies to 
\[
Q2(\mathcal{D}, t, l) = \sum_{i \in [N]}
\sum_{j \in [M]} \mathbb{I}[y_i = l] \cdot \prod_{n \in [N], n \ne i} \alpha_{i, j}[n].
\]

If we pre-compute the whole
$\alpha$ matrix, it is clear that 
a naive implementation would calculate the above value 
in $\mathcal{O}(N^2M)$.
However,
as we will see later, we can do much better.


\paragraph*{Efficient Implementation}
We can design a much more efficient 
algorithm to calculate this value.
The idea is to first sort all
$x_{i, j}$ pairs by their similarity
to $t$, $s_{i, j}$, from the smallest to the largest, and then scan them in this order.
In this way, we
can incrementally maintain the 
$\alpha_{i, j}$ vector during the scan.

Let $(i, j)$ be the current candidate 
value being scanned, and 
$(i', j')$ be the candidate 
value right before $(i, j)$ in the sort order, we have
\begin{equation}
\label{eqn:alpha}
\begin{split}
\alpha_{i, j}[n] & = \begin{cases}
\alpha_{i', j'}[n] + 1~~~~\text{if}~n = i', \\
\alpha_{i', j'}[n]. 
\end{cases}    
\end{split}
\end{equation}
Therefore, we are able to 
compute, for each $(i, j)$, its
\begin{equation}
\label{eqn:alphaprod}
\prod\nolimits_{n \in [N], n \ne i} \alpha_{i, j}[n]   
\end{equation}
in $\mathcal{O}(1)$ time, {\em without
pre-computing the whole $\alpha$.}
This will give us an algorithm 
with complexity $\mathcal{O}(MN\log MN)$!

\begin{example}
In Figure~\ref{fig:alg-sort-count-k1} we depict exactly this algorithm. We iterate over the candidates $x_{i,j}$ in
an order of increasing similarity with the test example $t$ (Step 1). In each iteration we try to compute the number of possible worlds supporting $x_{i,j}$ to be the top-$1$ data point (Step 2). We update the tally vector $\alpha$ according to Equation~\ref{eqn:alpha} (Step 3) and multiply its elements according to Equation~\ref{eqn:alphaprod} (Step 4) to obtain the boundary cont. Since $K=1$, the label support for the label $l=y_i$ is trivially equal to the boundary count and zero for $l \neq y_i$ (Step 5). We can see that the label $0$ is supported by $2$ possible worlds when $x_3=x_{3,1}$ and $4$ possible worlds when $x_3=x_{3,2}$. On the other hand, label $1$ has non-zero support only when $x_1=x_{1,2}$. Finally, the number of possible worlds that will predict label $l$ is obtained by summing up all the label supports in each iteration where $l=y_i$ (Step 6). For label $0$ this number is $2+4=6$, and for label $1$ it is $0+0+0+2=2$.
\end{example}


\subsubsection{$K \geq 1$, $|\mathcal{Y}| \geq 2$}

In the general case, the algorithm follows a 
similar intuition as the case of
$K = 1$ and $|\mathcal{Y}| = 2$. 
We enumerate each possible candidate value
$x_{i, j}$. For each candidate value,
we enumerate all possible values of the
label tally vector; for each such vector,
we compute its support. 
Let $\Gamma$ be the set of all possible 
label tally vectors, we have
\[
Q2(\mathcal{D}, t, l) =
   \sum_{i \in [N]} \sum_{j \in [M]}
   \sum_{\gamma \in \Gamma}
     \mathbb{I}[l = \argmax(\gamma)] 
     \cdot    
     Support (i, j, \gamma).
\]
We know that there are $|\Gamma| = \mathcal{O}\left({|\mathcal{Y}|+K-1 \choose K}\right)$
many possible configurations of the label 
tally vector, and for each of them,
we can compute the support 
$Support (i, j, \gamma)$ in $\mathcal{O}(NM|\mathcal{Y}|)$
time. As a result, a naive implementation of
the above algorithm would take 
$\mathcal{O}(N^2M^2|\mathcal{Y}|{|\mathcal{Y}|+K-1 \choose K})$ time.

\begin{algorithm}[t!]
\small
\caption{Algorithm \textbf{SS} for Answering Q2 with $K$-NN.} \label{alg:sort-count:dp:binary}
\label{alg:q2}
\begin{algorithmic}[1]
\REQUIRE {$\mathcal{D}$, incomplete dataset; $t$, target data point.}
\ENSURE {$r$, integer vector, s.t. $r[y]=Q2(\mathcal{D}, t, y), \forall y \in \mathcal{Y}$.}

\STATE $s \gets \mathtt{kernel} (\mathcal{D}, t)$;
\STATE $\alpha \gets \emph{zeros}(|\mathcal{D}|)$;
\STATE $r \gets \emph{zeros}(|\mathcal{Y}|)$;
\FORALL{$(i,j) \in \mathtt{argsort}(s)$}
    \STATE $\alpha[i] \gets \alpha[i] + 1$; ~~~~~// (See Equation~(\ref{eqn:alpha}))
    
    \FORALL{$l \in \mathcal{Y}$}
        \FORALL{$k \in [K]$}
        \STATE Compute $C^{i, j}_l(k, N)$.
        \ENDFOR
    \ENDFOR
    
    \FORALL{possible valid tally vectors $\gamma \in \Gamma$}
        \STATE $y_p \gets \mathtt{argmax} (\gamma)$;
        \STATE Compute 
        $Support(i, j, \gamma) = \mathbb{I}[\gamma_{y_i} \geq 1] \cdot \prod_{l \in \mathcal{Y}}C^{i, j}_l(\gamma_l, N)$;
        \STATE $r[y_p] \gets r[y_p] + Support(i, j, \gamma)$;
    \ENDFOR

\ENDFOR

\RETURN $r$;
\end{algorithmic}
\end{algorithm}

\paragraph*{Efficient Implementation}
We can implement the above procedure in a 
more efficient way, as illustrated in 
Algorithm~\ref{alg:q2}.
Similar to the case of $K=1$, we
iterate over all values $x_{i, j}$ in 
the order of increasing similarity (line 4). 
This way, we are able to maintain, efficiently,
the similarity tally vector $\alpha_{i, j}$ 
(line 5).
We then pre-compute the result of 
$K|\mathcal{Y}|$ many dynamic programming
procedures (lines 6-8), which will be used to
compute the support for each possible
tally vector later. 
We iterate over all valid label tally vectors, where a valid tally vector 
$\gamma \in \Gamma$ contains all integer vectors 
whose entries sum up to $K$ (line 9).
For each tally vector, we
get its prediction $y_p$ (line 10).
We then calculate its support (line 11) and
add it to the number of possible worlds
with $y_p$ as the prediction (line 12).

\vspace{0.5em}
\noindent
\textbf{(Complexity)}
We analyze the complexity of Algorithm~\ref{alg:q2}:
\begin{itemize}
\item The sorting procedure requires $\mathcal{O} (N \cdot M \log N \cdot M)$ steps as it sorts all elements of $S$.
\item The outer loop iterates over $\mathcal{O} (N \cdot M)$ elements.
\item In each inner iteration, we need to compute $|\mathcal{Y}|$
sets of dynamic programs, each of which
has a combined state space of size $N \cdot K$.
\item Furthermore, in each iteration, we iterate over all possible label assignments, which requires $\mathcal{O}\left({|\mathcal{Y}|+K-1 \choose K}\right)$ operations. 
\item For each label assignment, we need $\mathcal{O}(|\mathcal{Y}|)$ multiplications.
\end{itemize}
The time complexity is therefore the sum of $\mathcal{O} (N \cdot M \log (N \cdot M))$ and $\mathcal{O} (N \cdot M \cdot (N \cdot K + |\mathcal{Y}| + {|\mathcal{Y}|+K-1 \choose K}))$. 

\paragraph*{Further Optimizations}
We can make this even faster by observing that: (1) all the states relevant for each iteration of the outer loop are stored in $\alpha$, and (2) between two iterations, only one element of $\alpha$ is updated. We can take advantage of these observations to reduce the cost of computing the dynamic program by employing 
divide-and-conquer. We recursively divide the elements of $\alpha$ into two subsets and maintain the DP result for each subset. The joint result for the two subsets is obtained by a simple sum-of-products formula with $O(K)$ complexity. We can see that this enables us to maintain a binary tree structure of DP results and in each iteration we need to update $O(\log N)$ elements. This enables us to compute the dynamic program in $O(K \log N)$ instead of $O(KN)$ time, which renders the overall complexity as $\mathcal{O} (N \cdot M \cdot (\log(N\cdot M) +  K\update{^2} \cdot \log N ))$. 
We leave the details for the appendix.

\subsection{MM Algorithm}\label{sec:mm}

One can do significantly better for $Q_1$ in certain cases.
Instead of using the SS algorithm, we can develop an algorithm that deals with the binary 
classification case ($|\mathcal{Y}|=2$) with time complexity
$\mathcal{O} \left(N \cdot M + (N \log K + K) \right)$. 

This algorithm, illustrated in Figure~\ref{fig:alg-min-max-k1}, relies on a key observation that for each label $l$, we can greedily construct a possible world that has the best chance of predicting label $l$. We call this possible world the $l$-extreme world and construct it by selecting from each candidate set $\mathcal{C}_i$ either the candidate most similar to the test example $t$ when $y_i=l$, or the candidate least similar to $t$ when $y_i \neq l$. We can show
that the $l$-extreme world predicts label $l$
if and only if there exists a possible world that predicts label $l$.
This means we can use it as a condition for checking the possibility of predicting label $l$. Since the construction of the $l$-extreme world can be done in $\mathcal{O} \left(N \cdot M \right)$ time, this leads us to a efficient algorithm for $Q1$.

\begin{figure}
    \centering
    \includegraphics[width=0.7\textwidth]{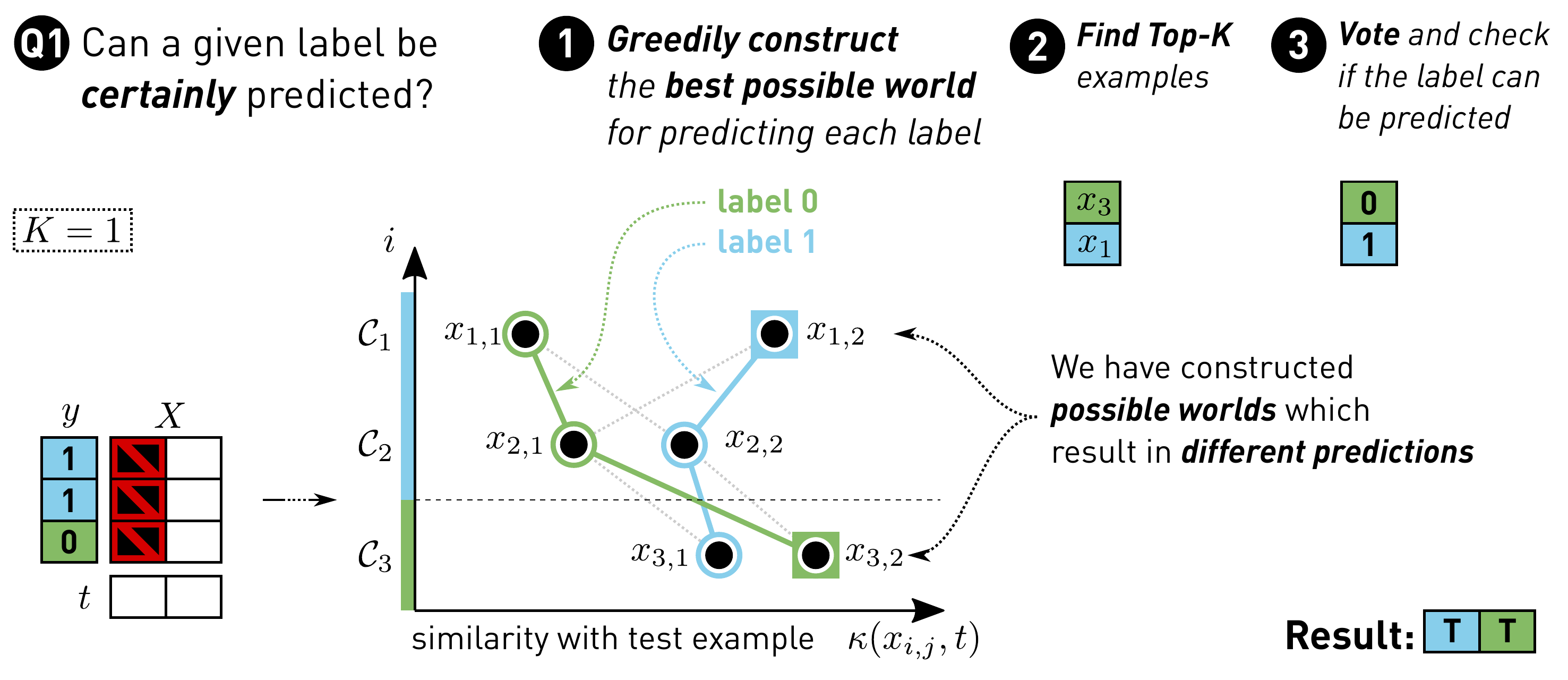}
    \caption{Illustration of MM when $K=1$ for $Q1$.}
    \label{fig:alg-min-max-k1}
\end{figure}


We first
describe the key idea behind 
the MM algorithm, and then 
describe the MM algorithm
which is listed in Algorithm~\ref{alg:min-max:k-nn}.

\paragraph*{Key Idea}
For \textit{binary classification} ($|\mathcal{Y}| = 2$), we have the following
observation --- given a possible world
$D = \{(x_{i, j_{i, D}}, y_i)\}$
that produces prediction $l \in \mathcal{Y}$
with a top-K set $Top(K, D, t)$,
consider a different possible 
world, which we call the
{\em $l$-extreme world} of $D$
as $E_{l, D}$. In $E_{l, D}$,
we replace,
for all candidate sets with 
$y_i = l$, the 
$x_{i, j_{i, D}}$ candidate in 
$D$ with the candidate in the candidate 
set $\mathcal{C}_i$ that is most 
similar to the test example
\[
j_{i, E_{l, D}} = \arg\max_j \kappa(x_{i, j}, t)
\]
and replace, for all candidate sets
with $y_i \ne l$,
the 
$x_{i, j_{i, D}}$ candidate in 
$D$ with the candidate in the candidate 
set $\mathcal{C}_i$ that is least 
similar to the test example
\[
j_{i, E_{l, D}} = \arg\min_j \kappa(x_{i, j}, t).
\]
We have
\[
D~\text{predicts}~l \implies E_{l, D}~\text{predicts}~l.
\]
To see why, note that (1) 
replacing all candidate values for
candidate set whose label $y_i \ne l$
by something less similar to the
test example $t$ will only make it
more likely to predict $l$; (2) 
replacing all candidate
values for candidate set whose label $y_i = l$
by something more similar to $t$ will only make it
more likely to predict $l$.

Another powerful observation is that
for all possible worlds $D$ they all 
have the \textit{same} $l$-extreme
worlds $E_{l, D}$ since the construction
of the latter only relies on the most
and least similar items in each candidate
sets. We can then write 
$E_l$ as the $l$-extreme world
for all possible world $D$.
We now have
\[
\exists D.~D~\text{predicts}~l \implies E_{l}~\text{predicts}~l,
\]
and, trivially
\[
E_{l}~\text{predicts}~l  \implies \exists D.~D~\text{predicts}~l, 
\]
by simply taking $D = E_{l}$. As a result,
\[
\exists D.~D~\text{predicts}~l \Leftrightarrow E_{l}~\text{predicts}~l.
\]
One can use this observation to check whether 
$Q1(\mathcal{D}, t, l)$ evaluates to true:
this is equivalent to checking 
whether there exists any possible world
$D$ that predicts a label $l' \ne l$.
To achieve this, we can simply check
the $l'$- extreme world $E_{l'}$.

\begin{algorithm}[t!]
\small
\caption{Algorithm \textbf{MM} for answering Q1 with $K$-NN.} \label{alg:min-max:k-nn}
\begin{algorithmic}[1]
\REQUIRE {$\mathcal{D}$, incomplete dataset; $t$, target data point.}
\ENSURE {$r$, Boolean vector, s.t. $r[y]=Q1(\mathcal{D}, t, y), \forall y \in \mathcal{Y}$.}

\STATE $S \gets \mathtt{kernel} (\mathcal{D}, x_t)$

\FORALL{$i \in 1,...,|\mathcal{D}|$}
    \STATE $s_i^{\min} \gets \min\{S_{i,j}\}_{j=1}^M$, $s_i^{\max} \gets \max\{S_{i,j}\}_{j=1}^M$;
\ENDFOR

\FORALL{$l\in\mathcal{Y}$}
    \STATE $s \gets \mathtt{zeros} ( |\mathcal{D}|)$;
    \FORALL{$i \in 1,...,|\mathcal{D}|$}
        \STATE $s[i]$ $\gets$ $s_i^{\max}$ $\mathbf{if}$ $(y_i = l)$ $\mathbf{else}$ $s_i^{\min}$;
    \ENDFOR
    
    \STATE $I_K \gets \mathtt{argmax\_k} (s, K)$;
    \STATE $v \gets \mathtt{vote}( \{  y_i : i \in I_K  \} ) $;
    \IF{$\mathtt{argmax} (v) = l$}
        \STATE $r[l] \gets \mathtt{true}$;
    \ELSE
        \STATE $r[l] \gets \mathtt{false}$;
    \ENDIF
\ENDFOR

\RETURN $r$;

\end{algorithmic}
\end{algorithm}




    



\paragraph*{Proof in Appendix} 
This idea might look simple
and natural, however, a formal proof 
is actually quite 
engaged (e.g., without a formal proof, it is not 
immediately clear why this algorithm cannot
handle cases in which $|\mathcal{Y}| > 2$).
We leave the full, formal proof to the
appendix of this paper.

\paragraph*{Efficient Algorithm} 
The above intuition gives us a very efficient 
algorithm to answer the query $Q1(\mathcal{D}, t, l)$,
as illustrated in Algorithm~\ref{alg:min-max:k-nn}.
We first calculate 
the similarity matrix (line 1), compute the extreme similarities that we use later (lines 2-3), and then
try to construct the $l$-extreme world
for each $l \in \mathcal{Y}$ (lines 4-7).
We then calculate the top-K set
of the $l$-extreme world (line 8)
and tally the labels to get the prediction
in the $l$-extreme world (line 9).

To answer the query $Q1(\mathcal{D}, t, l)$
(lines 10-13),
we check all $l'$-extreme worlds where $l' \ne l$
to see if any of these $l'$-extreme worlds 
predicts their corresponding $l'$. If yes,
then $Q1(\mathcal{D}, t, l) = \texttt{false}$;
otherwise, $Q1(\mathcal{D}, t, l) = \texttt{true}$.

\vspace{0.5em}
\noindent
\textbf{(Complexity)}
We analyze the complexity of Algorithm~\ref{alg:min-max:k-nn} as follows:
\begin{itemize}
\item We first precompute the similarity matrix, as well as the minimum and maximum similarities, both of which can be done in $\mathcal{O}(N \cdot M)$ time.
\item The outer loop is executed $|\mathcal{Y}|$ times.
\item The optimal world construction loop (lines 6-7) is executed $N$ times. In each iteration we retrieve the precomputed maximal or minimal values.
\item The $\mathtt{argmax\_k}$ function implemented as a heap requires $\mathcal{O} (N \log K)$ steps.
\item The $\mathtt{vote}$ function requires $\mathcal{O}(K)$ steps. The $\mathtt{argmax}$ takes $\mathcal{O}(|\mathcal{Y}|)$, although these two steps can be implemented jointly and run in $\mathcal{O}(K)$ time.
\end{itemize}
The time complexity is therefore $\mathcal{O} \left(N \cdot M + |\mathcal{Y}| \cdot (N \log K + K) \right)$.

\section{Application: Data Cleaning for ML}\label{sec:cc:general}

In this section, we show how to use the proposed CP framework to design an effective data cleaning solution, called \at{CPClean}, for the important application of data cleaning for ML. We assume as input a dirty training set $\mathcal{D}_{train}$ with unknown ground truth $D^*$ among all
possible worlds $\mathcal{I}_\mathcal{D}$. 
Our goal is to select a version $D$ from $\mathcal{I}_\mathcal{D}$, such that
the classifier trained on 
$\mathcal{A}_D$ has the 
same validation accuracy as the classifier trained on the ground
truth world $\mathcal{A}_{D^*}$. 

\stitle{Cleaning Model.} 
Given a dirty dataset $\mathcal{D} = \{(\mathcal{C}_i, y_i)\}_{i\in[N]}$,
in this paper, we focus on the 
scenario in which 
the candidate set $\mathcal{C}_i$
for each data example
is created by automatic data 
cleaning algorithms or
a predefined noise model. For each 
uncertain data example $\mathcal{C}_i$,
we can ask a human 
to provide its true value $x_i^* \in \mathcal{C}_i$. Our goal is 
to find a good strategy to prioritize 
which dirty examples to be cleaned. That is, a cleaning strategy of
$T$ steps can be defined as
\[
\pi \in [N]^T,
\]
which means that in the first
iteration, we clean the example 
$\pi_1$ (by querying human to obtain the ground
 truth value of $\mathcal{C}_{\pi_1}$;
 in the second iteration, we
 clean the example 
$\pi_2$; and so on.
Applying a cleaning strategy $\pi$ will generate  a \textit{partially
cleaned} dataset $\mathcal{D}_{\pi}$
in which all cleaned candidate
sets $\mathcal{C}_{\pi_i}$
are replaced by $\{x_{\pi_i}^*\}$.


\stitle{Formal Cleaning Problem Formulation.} The question we need to address is \textit{"What is a successful cleaning
strategy?"} Given a validation set
$D_{val}$, the view of \at{CPClean} 
is that a successful cleaning strategy
$\pi$ should be the one that 
produces a partially cleaned 
dataset $\mathcal{D}_\pi$ in which 
all validation examples $t \in D_{val}$
can be certainly predicted.
In this case, picking any
possible world defined by 
$\mathcal{D}_\pi$, i.e., $\mathcal{I}_{\mathcal{D}_\pi}$,
will give us a dataset that
has the same accuracy, on the
validation set, as the ground truth world
$D^*$. This can be defined
 precisely as follows.

We treat each candidate set
$\mathcal{C}_i$ as a random variable
$\mathbf{c}_i$, taking values in 
$\{x_{i, 1}, ..., x_{i, M}\}$. 
We write $\mathbf{D} = \{
(\mathbf{c}_i, y_i)\}_{i \in [N]}$.
Given a cleaning strategy $\pi$
we can define the conditional entropy
of the classifier prediction 
on the validation set as

\begin{equation}
\mathcal{H}(
\mathcal{A}_{\mathbf{D}}(D_{val})| 
\mathbf{c}_{\pi_1}, ...,
\mathbf{c}_{\pi_T}
)
\coloneqq 
\frac{1}{|D_{val}|}\sum_{t \in D_{val}}
\mathcal{H}(
\mathcal{A}_{\mathbf{D}}(t)| 
\mathbf{c}_{\pi_1} , ...,
\mathbf{c}_{\pi_T} 
).
\label{eqn:conditional_entropy}
\end{equation}

Naturally, this gives us a principled
objective for finding a ``good''
cleaning strategy that minimizes the human cleaning effort: 
\begin{eqnarray*}
& \min_\pi & \text{dim}(\pi)\\
& s.t.,  &
\mathcal{H}(
\mathcal{A}_{\mathbf{D}}(D_{val})| 
\mathbf{c}_{\pi_1} = x_{\pi_1}^*, ...,
\mathbf{c}_{\pi_T} = x_{\pi_T}^*
) = 0.
\end{eqnarray*}
If we are able to find a cleaning 
strategy in which 
\[
\mathcal{H}(
\mathcal{A}_{\mathbf{D}}(D_{val})| 
\mathbf{c}_{\pi_1} = x_{\pi_1}^*, ...,
\mathbf{c}_{\pi_T} = x_{\pi_T}^*
) = 0,
\]
we know that this 
strategy would produce a partially 
cleaned dataset $\mathcal{D}_{\pi}$
on which all validation examples
can be CP'ed.
Note  that we can use 
the query $Q2$ to compute this conditional 
entropy:
\begin{equation*}
\mathcal{H}(
\mathcal{A}_{\mathbf{D}}(t)| 
\mathbf{c}_{\pi_1} = x_{\pi_1}^*, ...,
\mathbf{c}_{\pi_T} = x_{\pi_T}^*
)
~ =
-
\sum_{l \in \mathcal{Y}}
\frac{Q2(\mathcal{D}_\pi, t, y)}{|\mathcal{D}_\pi|} \log \frac{Q2(\mathcal{D}_\pi, t, y)}{|\mathcal{D}_\pi|}
\end{equation*}

\vspace{0.5em}
\noindent
{\bf Connections to ActiveClean.} 
The idea of prioritizing 
human cleaning effort for
downstream ML models is not new ---
ActiveClean~\cite{krishnan2016activeclean} explores
an idea with a similar goal. 
However, there are some
important differences 
between our framework and 
ActiveClean. The most 
crucial one is that 
our framework relies on 
\textit{consistency} of predictions
instead of the \textit{gradient},
and therefore, we do not
need labels for the
validation set and our algorithm can be used 
in ML models that cannot be 
trained by gradient-based
methods. The KNN classifier
is one such example. 
Since both frameworks 
essentially measure some
notion of ``local sensitivity,''
it is interesting future work 
to understand how to combine them.

\subsection{The CPClean Algorithm} 
\label{sec:seq-info-max}

Finding the solution to the
above objective is, not
surprisingly, NP-hard~\cite{DBLP:journals/ior/KoLQ95}. In this paper, 
we take the view of 
sequential information maximization introduced by~\cite{Chen2015SequentialIM}
and \update{adapt} the \update{respective} greedy algorithm
for this problem. We first describe 
the algorithm, and then \update{review}
the theoretical analysis of its behavior.

\stitle{Principle: Sequential Information Maximization.}
Our goal is to find a cleaning
strategy that \textit{minimizes} the 
\textit{conditional entropy} as fast 
as possible. An equivalent view of 
this is to find a cleaning
strategy that \textit{maximizes}
the \textit{mutual information}
as fast as possible. 
While we use the view of minimizing conditional entropy in implementing the CPClean algorithm, the equivalent view of maximizing mutual information will be useful in analyzing theoretical guarantees about CPClean.

Given 
the current $T$-step cleaning
strategy $\pi_1, ... ,\pi_T$, our goal 
is to greedily find the \textit{next}
data example to clean $\pi_{T+1} \in [N]$
that minimizes the
entropy conditioned on the partial observation as fast as possible:
\[
\pi_{T+1} = \arg\min_{i \in [N]}
\mathcal{H}(
\mathcal{A}_{\mathbf{D}}(D_{val})| 
\mathbf{c}_{\pi_1} = x_{\pi_1}^*, ...,
\mathbf{c}_{\pi_T} = x_{\pi_T}^*,
\mathbf{c}_{i} = x_{i}^*).
\]

\noindent
{\bf Practical Estimation.} The question thus becomes how to estimate
\[
\mathcal{H}(
\mathcal{A}_{\mathbf{D}}(D_{val})| 
\mathbf{c}_{\pi_1} = x_{\pi_1}^*, ...,
\mathbf{c}_{\pi_T} = x_{\pi_T}^*,
\mathbf{c}_{i} = x_{i}^*) ?
\]
The challenge is that when we are trying to decide which example to clean, we do not know the ground truth for item $i$, $x_i^*$. As a result, we need to
assume some priors on how likely
each candidate value $x_{i,j}$ 
is the ground truth $x_i^*$.
In practice, we find that a 
uniform prior already works well;
this leads to the following 
expected value:
\begin{align}
\frac{1}{M} \sum_{j \in [M]}
\mathcal{H}(
\mathcal{A}_{\mathbf{D}}(D_{val})| 
\mathbf{c}_{\pi_1} = x_{\pi_1}^*, ...,
\mathbf{c}_{\pi_T} = x_{\pi_T}^*,
\mathbf{c}_{i} = x_{i, j}).
\label{eqn:expected_entropy}
\end{align}
The above term can thus be calculated by invoking the $Q2$ query.

\begin{algorithm}[t!]
\small
\caption{Algorithm \at{CPClean}.} \label{alg:ccclean}
\begin{algorithmic}[1]
\REQUIRE {$\mathcal{D}$, incomplete training set;  $D_{val}$, validation set.}
\ENSURE {$D$, a dataset in 
$\mathcal{I_{\mathcal{D}}}$ s.t. $\mathcal{A}_{D}$ and $\mathcal{A}_{D^*}$ have same validation accuracy}
\STATE $\pi \gets [ ]$

\FOR{$T = 0$ to $N-1$}

\IF{$D_{val}$ all CP'ed}
\STATE \textbf{break}
\ENDIF
    \STATE $min\_entropy \gets \infty$
    \FORALL{$i \in [N] \backslash \pi $}
    \STATE $entropy = \frac{1}{M} \sum\limits_{j \in [M] }  \mathcal{H}(
\mathcal{A}_{\mathbf{D}}(D_{val})| 
\mathbf{c}_{\pi_1} = x_{\pi_1}^*, ..., \mathbf{c}_{\pi_{T}} = x_{\pi_{T}}^*,\mathbf{c}_{i} = x_{i, j})$ 
    %
    \IF{$entropy < min\_entropy$}
        \STATE $\pi_{T+1} \gets i$, $min\_entropy \gets entropy$
    \ENDIF
    \ENDFOR
    \STATE $\mathbf{x}_{\pi_{T+1}}^* \gets$ obtain the ground truth of $\mathcal{C}_{\pi_{T+1}}$ by human

\ENDFOR
\RETURN Any world $D \in \mathcal{I}_{\mathcal{D}_{\pi}}$ 

\end{algorithmic}
\end{algorithm}

\begin{figure}[t!]
    \centering
    \includegraphics[width=0.5\textwidth]{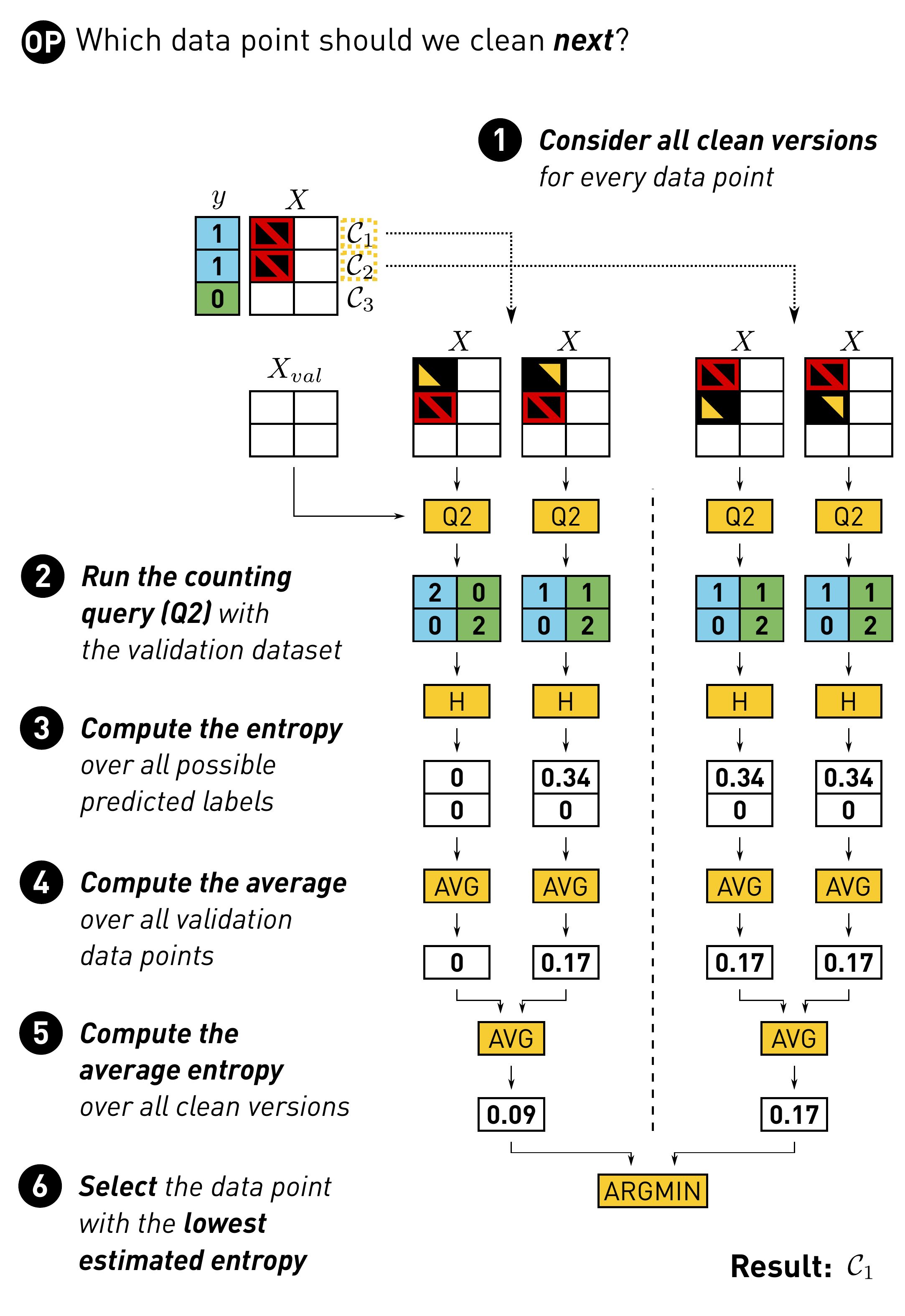}
    \caption{CPClean via sequential info. maximization.}
    \label{fig:seq-info-max}
\end{figure}

\vspace{0.5em}
\noindent
{\bf CPClean.} The pseudocode for \at{CPClean} is shown in Algorithm \ref{alg:ccclean}. The algorithm starts with an empty cleaning strategy (line 1). In each iteration, given the current cleaning strategy $\pi_1, ..., \pi_{T}$, we compute the expected value of entropy conditioned on cleaning one extra training example (lines 6-7). We select the next example to clean $\pi_{T+1}$ that minimizes the entropy (lines 8-9). We then ask a human to clean the selected example (line 10). The greedy algorithm terminates when all validation examples become CP'ed (line 3). Finally, we return any world $D$ among all possible partially cleaned worlds $\mathcal{I}_{\mathcal{D}_{\pi}}$ (line 12). Since all the validation examples are CP'ed with $\mathcal{I}_{\mathcal{D}_{\pi}}$, classifier trained on any world in $\mathcal{I}_{\mathcal{D}_{\pi}}$, including the unknown ground truth world $D^*$, has the same validation accuracy. Therefore, $\mathcal{A}_{D}$ has the same validation accuracy as $\mathcal{A}_{D^*}$.


\begin{example}  Figure~\ref{fig:seq-info-max} shows an example of how \at{CPClean} selects the next data example to clean in each iteration via sequential information maximization. Assume there are two dirty examples, $\mathcal{C}_1$ and $\mathcal{C}_2$, in the training set and each example has two candidate repairs. Therefore, there are four possible clean versions after cleaning the next data point, based on which data point is selected to be cleaned and which candidate repair is the ground truth. For example, the first table at step 1 shows the clean version after cleaning $
\mathcal{C}_1$ if $x_{1, 1}$ is the ground truth. Assume that we have two validation examples. We run the counting query (Q2) on each possible version w.r.t. each validation example as shown in step 2. Then we can compute the entropy of predictions on validation examples as shown in step 3 and 4. The results show that if $\mathcal{C}_1$ is selected to be cleaned, the entropy may become $0$ or $0.17$ depending on which candidate repair is the ground truth. We assume that each of the two candidate repairs has 50\% chance to be the ground truth. Therefore, the expected entropy after cleaning $\mathcal{C}_1$ is $(0 + 0.17) /2 = 0.09$ (step 5). Similarly, we compute the expected entropy after cleaning $\mathcal{C}_2$ as $0.17$. Since $\mathcal{C}_1$ has a lower expected entropy, we select $\mathcal{C}_1$ to clean. 



\end{example}

\noindent
{\bf Complexity of CPClean.} In each iteration of Algorithm~\ref{alg:ccclean}, we need to (1) automatically select a tuple; and (2) ask human to clean the selected tuple. To select a tuple, we need to first check  whether $|D_{val}|$ are all CP'ed (line 3), which invokes the $Q1$ query $O(|D_{val}|)$ times. If not all $D_{val}$ are CP'ed, we need to compute expected value of entropy $O(N)$ times (line 6). Computing the expected value of entropy (line 7) needs to invoke the $Q2$ query $O(M|D_{val}|)$ times. Therefore, when the downstream ML model is KNN, using our SS algorithm for $Q1$ and $Q2$, the complexity for selecting a tuple at each iteration is $O(N^2M^2|D_{val}| \times(\log (MN) + K\log N))$. The quadratic complexity in tuple selection is acceptable in practice, since  human involvement is generally considered to be the most time consuming part in practical data cleaning~\cite{DBLP:books/acm/IlyasC19}.

\vspace{0.5em}
\noindent
{\bf Theoretical Guarantee.} The theoretical analysis of this algorithm, while resembling that of~\cite{Chen2015SequentialIM}, is non-trivial. We provide the main \update{theoretical analysis} here and leave the proof to the appendix.

\begin{corollary}\label{Corollary: theoretical guarantee of cleaning}
Let the optimal cleaning policy that minimizes the cleaning effort while consistently classifying the test examples be denoted by $D_{\text{Opt}}\subseteq D_{train}$ with limited cardinality $t$,
such that 
$$D_{\text{Opt}} =  \argmax_{D_{\pi} \subseteq \mathcal{D}_{train}, \ |D_{\pi}|\leq t}  I(\mathcal{A}_{\mathbf{D}}(D_{val});
D_{\pi}).$$ 
The sequential information maximization strategy 
follows a near optimal strategy where the information gathering satisfies
\begin{equation*}
\begin{split}
I(\mathcal{A}_{\mathbf{D}}(D_{val});
\mathbf{c}_{\pi_1}, ...,
\mathbf{c}_{\pi_T})
 \geq I(\mathcal{A}_{\mathbf{D}}(D_{val});
D_{\text{Opt}})(1-\exp{(-{T}/{\theta t'})})
\end{split}
\end{equation*}
where 
\begin{align*}
 &\theta=\left(\max_{v\in \mathcal{D}_{\text{train}}}I(\mathcal{A}_{\mathbf{D}}(D_{val}); v)\right)^{-1}
\\
t'=t\min\vspace{-1em}\{&\log |\mathcal{Y}|, \log M \}, \ \ \mathcal{Y}:\text{label space}, \ M: |\mathcal{C}_i|. 
\end{align*}
\end{corollary}
The above result, similarly as in~\cite{Chen2015SequentialIM}, suggests that data cleaning is guaranteed to achieve near-optimal information gathering up to a logarithmic factor $\min(\log |\mathcal{Y}|, \log M)$ when leveraging the sequential information strategy.

\section{Experiments}\label{sec:evaluation}

We now conduct an extensive set of experiments to compare CPClean with other data cleaning approaches in the context of K-nearest neighbor classifiers.


\begin{table}[t!]
\centering
\scalebox{0.8}{
\begin{tabular}{l|cccc}
\hline
\textbf{Dataset} & \textbf{Error Type} & \textbf{\#Examples} & \textbf{\#Features}&\textbf{Missing rate} \\ \hline
\at{BabyProduct}~\cite{babyproduct} & real &3042  &7 &11.8\% \\ 
\at{Supreme}~\cite{supreme} & synthetic &3052  &7&20\%  \\ 
\at{Bank}~\cite{DELVE} & synthetic & 3192 & 8&20\% \\ 
\at{Puma}~\cite{DELVE} & synthetic & 8192 & 8&20\% \\ \hline
\end{tabular}
}
\caption{Datasets characteristics}
\label{tbl:datasets}
\vspace{-5mm}
\end{table}

\subsection{Experimental Setup}

\stitle{Hardware and Platform.} All our experiments were performed on a machine with a 2.20GHz Intel Xeon(R) Gold 5120 CPU.

\stitle{Datasets.} One main challenge of evaluating data cleaning solutions is the lack of datasets with ground truth, and hence most data cleaning work resort to synthetic error injection. This is especially true in the context of incomplete information: a dataset with missing values is not likely to come with ground truth. In this work, besides three datasets with synthetic errors, we manage to find one dataset with \textit{real} missing values, where we are able to obtain the \textit{ground truth} via manual Googling. We
summarize all datasets in Table~\ref{tbl:datasets}. 

The \at{BabyProduct} dataset  contains various baby products of different categories (e.g., bedding, strollers).  Since the dataset was scraped from websites using Python scripts~\cite{babyproduct}, many records have missing values, presumably due to extractor errors.
We designed a classification task to predict whether a given baby product has a high price or low price based on other attributes (e.g. weight, brand, dimension, etc), and we selected a subset of product categories whose price difference is not so high so as to make the classification task more difficult.
For records with missing brand attribute, we then perform a Google search using the product title to obtain the product brand. For example, one record titled ``{\em Just Born Safe Sleep Collection Crib Bedding in Grey}'' is missing the product brand, and a search reveals that the brand is ``Just Born.''




We also use three datasets (\at{Supreme}, \at{Bank}, \at{Puma}), originally with no missing values, to inject synthetic missing values. Our goal is to inject missing values in the most realistic way possible and also to ensure that the missing values can have a large impact on classification accuracy. We follow the popular ``Missing Not At Random'' assumption~\cite{rubin1976inference}, where the probability of missing may be higher for more sensitive/important attributes. For example, high income people are more likely to not report their income in a survey. 
We first assess the relative importance of each feature in a classification task (by measuring the accuracy loss after removing a feature), and use the relative feature importance as the relative probability of a feature missing. We can then inject missing values into a dataset for any given missing rate (we use 20\% in our experiment).



\stitle{Model.} We use a KNN classifier with K=3
and use Euclidean distance 
as the similarity function.
%
For each dataset, we randomly select 1,000 examples as the validation set and 1,000 examples as the test set. The remaining examples are used as the training set.

\stitle{Cleaning Algorithms Compared.} We compare the following approaches for handling missing values in the training data.
\begin{itemize}[leftmargin=*]
   \item \underline{\emph{Ground Truth}}: This method uses the ground-truth version of the dirty data, and shows the performance upper-bound. 
    \item \underline{\emph{Default Cleaning}}: This is the default and most commonly used way for cleaning missing values in practice, namely, missing cells in a numerical column are filled in using the mean value of the column, and those in a categorical column are filled using the most frequent value of that column.
     \item \underline{\emph{CPClean}}: This is our proposal, which needs a candidate repair set $\mathcal{C}_i$ for each example with missing values. For missing cells in numerical columns, we consider five candidate repairs: the minimum value, the 25-th percentile, the mean value, the 75-th percentile and the maximum value of the column. For missing cells in categorical columns, we also consider five candidate repairs: the top 4 most frequent categories and a dummy category named ``other category". If a record $i$ has multiple missing values, then the Cartesian product of all candidate repairs for all missing cells forms $\mathcal{C}_i$. 
     We simulate human cleaning by picking the candidate repair that is closest to the ground truth.
    \item \underline{\emph{HoloClean}}: This is the state-of-the-art probabilistic data cleaning method~\cite{rekatsinas2017holoclean}. 
    As a weakly supervised machine learning system, it leverages multiple signals (e.g. quality rules, value correlations, reference data) to build a probabilistic model for imputing and cleaning data. Note that the focus of HoloClean is to find the most likely fix for a missing cell in a dataset without considering how the dataset is used by downstream classification tasks.  
    \item \underline{\emph{BoostClean}}: This is the state-of-the-art automatic data cleaning method for ML~\cite{krishnan2017boostclean}. At a high level, it selects, from a predefined set of cleaning methods, the one that has the maximum validation accuracy on the validation set. To ensure fair comparison, we use the same cleaning method as in CPClean, i.e., the predefined cleaning methods include cleaning a numerical column with missing values using 25-th percentile, the mean value, etc. We also use the same validation set as in CPClean.
    \item \underline{\emph{RandomClean}}: While CPClean uses the idea of sequential information maximization to select which examples to clean, \emph{RandomClean} simply selects an example randomly  to clean.  
    
    


\end{itemize}

\stitle{Performance Measures.} Besides the cleaning effort spent, we are mainly concerned with the test accuracy of models trained on datasets cleaned by different cleaning methods. Instead of reporting exact test accuracies for all methods, we only report them for \emph{Ground Truth} and \emph{Default Cleaning}, which represents the upper bound and the lower bound, respectively. For other methods, we report the percentage of closed gap defined as:
\begin{small}
    $$ \textit{gap closed by X} = \frac{\text{ accuracy(X) -  accuracy(\emph{Default Cleaning})}}{\text{ accuracy(\emph{Ground Truth}) -  accuracy(\emph{Default Cleaning})}}.$$
\end{small}{}




\subsection{Experimental Results}
\label{subsec:end-to-end}

\begin{table*}[!ht]
\centering
\scalebox{0.65}{
\begin{tabular}{c|cc|c|c|cc|cc}
\hline
\multirow{2}{*}{\textbf{Dataset}} & \emph{Ground Truth} & \emph{Default Cleaning} & \emph{BoostClean} & \emph{HoloClean} & \multicolumn{4}{c}{\emph{CPClean}}
\\ \cline{2-9} 
 & \textbf{Test Accuracy} & \textbf{Test Accuracy} & \textbf{Gap Closed} & \textbf{Gap Closed} & \textbf{Gap Closed } & \textbf{Examples Cleaned } 
 &
 \textbf{Gap Closed } &
 \textbf{Examples Cleaned }
 \\ \hline
BabyProduct & 0.668 & 0.589 & 1\% & 1\% & 99\% & 64\% & 72\% & 20\% \\ 
Supreme & 0.968 & 0.877 & 12\% & -4\% & 100\% & 15\% & 100\% & 20\% \\ 
Bank & 0.643 & 0.558 & 20\% & 11\% & 102\% & 93\% & 52\% & 20\% \\ 
Puma & 0.794 & 0.747 & 28\% & -64\% & 102\% & 63\% & 40\% & 20\%\\ \hline
\end{tabular}
}
\caption{End-to-End Performance Comparison}
\label{table:end_to_end_performance}
\end{table*}

\stitle{Model Accuracy Comparison.} Table \ref{table:end_to_end_performance} shows the end-to-end performance of our method and other automatic cleaning methods. We can see that the missing values exhibit different degress of impact on these datasets (the gap between \emph{Ground Truth} and \emph{Default Cleaning}). We can also observe that \emph{HoloClean}, the state-of-the-art standalone data cleaning approach performs poorly --- the gap closed can even be negative. This suggests that performing data cleaning on a data without considering how it is used later may not necessarily improve downstream model performance. On the other hand, we observe that \emph{BoostClean} shows a consistently positive impact on model performance by using the validation set to pick the most useful cleaning method. In all cases, \emph{CPClean} is able to close ~100\% of gap without manual cleaning of all dirty data. In fact, on \at{Supreme}, \emph{CPClean} only requires the manual cleaning of 15\% of missing records to close 100\% gap. We can also see from Table \ref{table:end_to_end_performance} that, by cleaning only 20\% of all dirty data, i.e., terminating the cleaning process at 20\% mark even if not all validation examples are CP'ed, \emph{CPClean} is able to close 66\% gap on average.

\begin{figure}[t!]
    \label{fig:vary_val}
    \centering
    \includegraphics[width=0.6\textwidth]{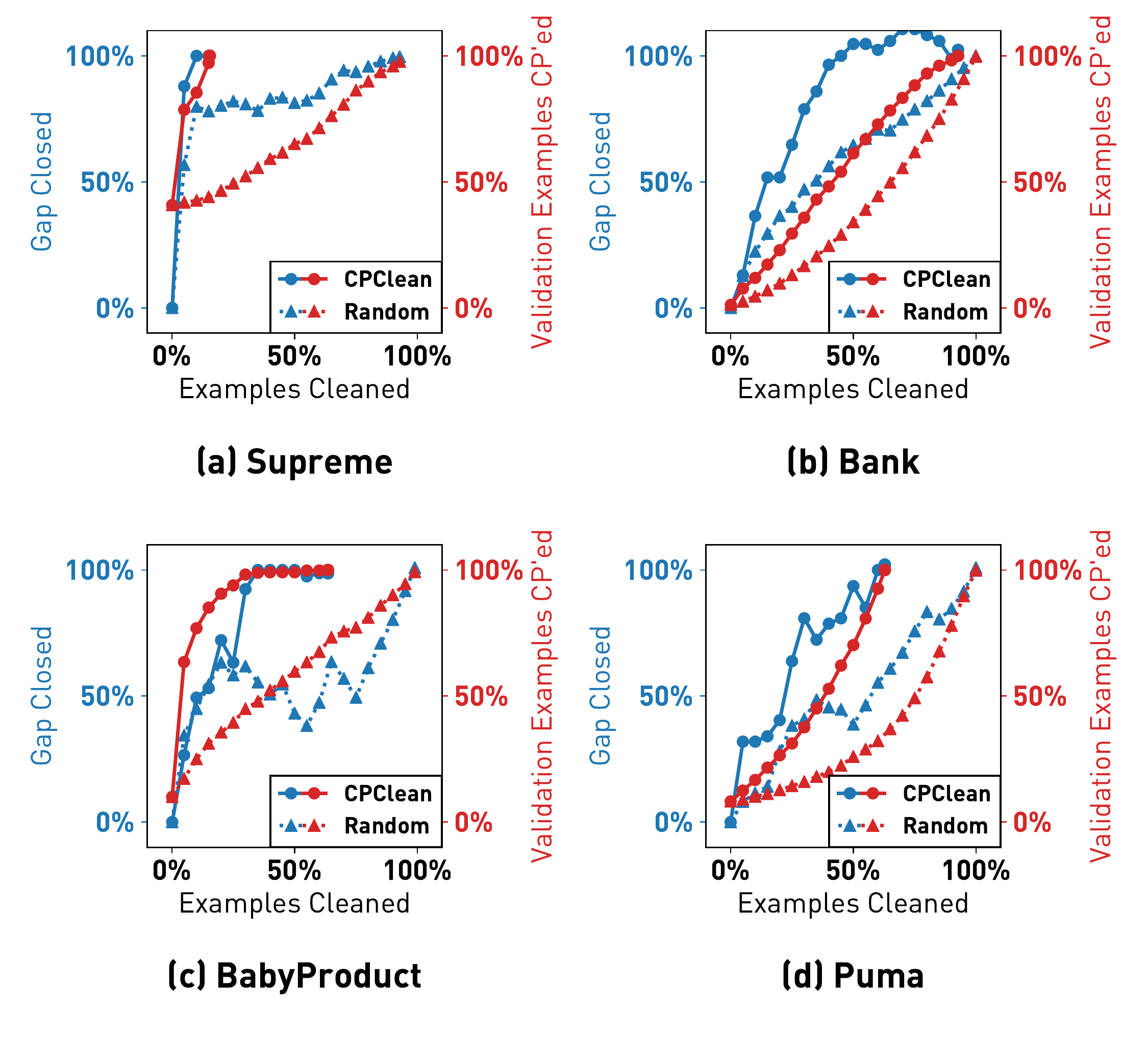}
    \caption{Comparison with Random Cleaning}
    \label{fig:cc_vs_random_cc_acc}
\end{figure}{}

\stitle{Early Termination.} If users have a limited cleaning budget, they may choose to terminate \emph{CPClean} early. To study the effectiveness of \emph{CPClean} in prioritizing cleaning effort, we compare it with \emph{RandomClean} that randomly picks an example to clean at each iteration. The results for \emph{RandomClean} are the average of 20 runs.

The red lines in Figure \ref{fig:cc_vs_random_cc_acc} show the percentage of CP'ed examples in the validation set as more and more examples are cleaned. As we can see, \emph{CPClean} (solid red line) dramatically outperforms the \emph{RandomClean} (dashed red line) both in terms of  the number of training examples cleaned so that all validation examples are CP'ed and in terms of the rate of convergence. For example, for \at{Supreme}, \emph{CPClean} requires the cleaning of 15\% examples while \emph{RandomClean} requires cleaning almost all training examples. 


The blue lines in Figure \ref{fig:cc_vs_random_cc_acc} show the percentage of gap closed for the test set accuracy. Again, we can observe that \emph{CPClean} significantly outperforms \emph{RandomClean}. 
For example, with 50\% of data cleaned in 
\at{Bank}, \emph{RandomClean} only closes about 65\% of the gap, whereas \emph{CPClean} closes almost 100\% of the gap. 

\begin{figure}[t!]
    \label{fig:vary_val}
    \centering
    \includegraphics[width=0.6\textwidth]{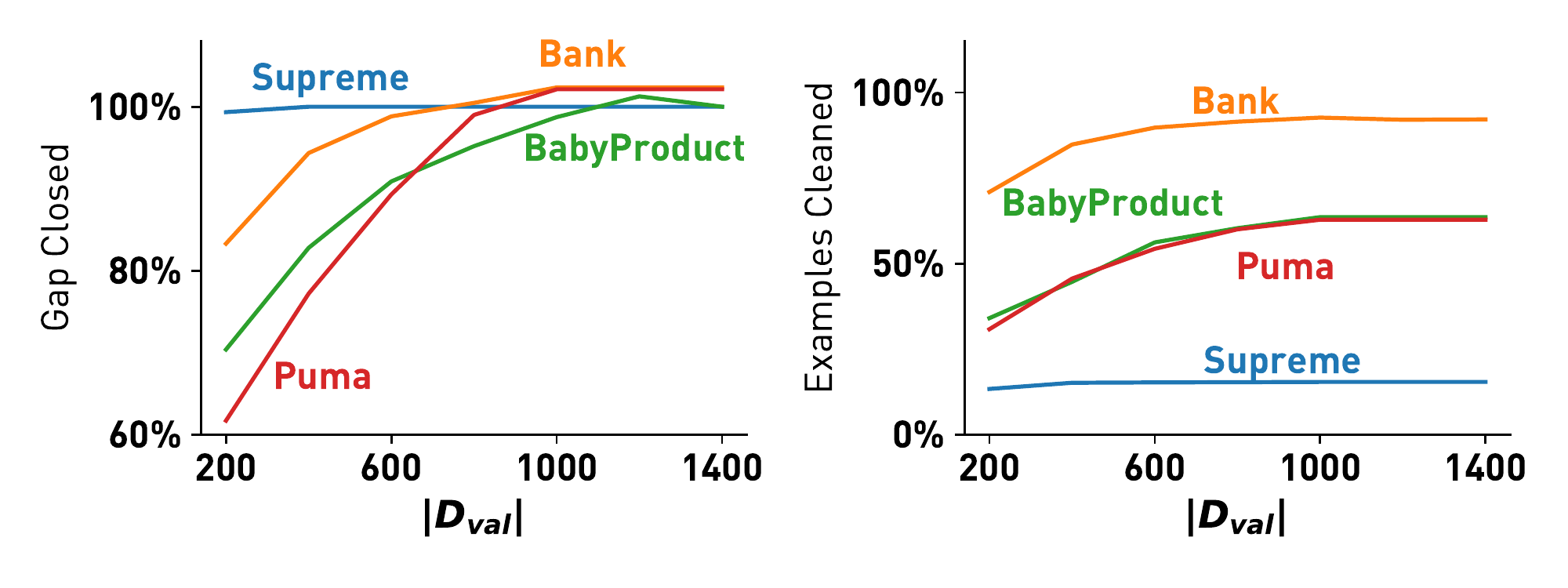}
    \caption{Varying size of $D_{val}$.}
    \label{fig:vary_val}
\end{figure}{}

\stitle{Size of the Validation Set $D_{val}$.}  We vary the size of $D_{val}$ to understand how it affects the result. As shown in Figure \ref{fig:vary_val}, as the size of validation set increases, both the test accuracy gap closed and the cleaning effort spent first increase and then become steady. This is because, when the validation set is small, it is easier to make all validation examples CP'ed (hence the smaller cleaning effort). However, a small validation set may not be representative of some unseen test set, and hence may not close the accuracy gap on test set. In all cases, we observe that 1K validation set is sufficiently large and further increasing it does not improve the performance.

\section{Related Work}\label{sec:related-work}

\stitle{Relational Query over Incomplete Information.} 
This work is heavily inspired by
the database literature of 
handling incomplete information~\cite{Alicebook},
consistent query answering~\cite{DBLP:conf/pods/ArenasBC99,DBLP:conf/icdt/LopatenkoB07,DBLP:series/synthesis/2011Bertossi}, and probabilistic databases~\cite{PDBBOOK}. 
While
these work targets SQL analytics, our proposed consistent prediction query targets ML analytics. 

\stitle{Learning over Incomplete Data.} The statistics and ML community have also studied the problem of learning over incomplete data. Many studies operate under certain missingness assumption (e.g., missing completeness at random) and reason about the performance of downstream classifiers in terms of asymptotic properties and in terms of
different imputation strategies~\cite{10.1007/s00521-009-0295-6}. 
In this work, we focus more on the
algorithmic aspect of this problem
and try to understand how to enable
more efficient manual cleaning
of the data.
Another flavor of work aims at developing ML models that are robust to certain types of noises, and multiple imputation~\cite{10.2307/2291635} is such a method that is most relevant to us. Our CP framework can be seen as an extreme case of multiple imputation (i.e, by trying all possible imputations) with efficient implementation (in KNN), which also enables novel manual cleaning for ML use cases. 

\update{Recently, Khosravi et al.
\cite{DBLP:journals/corr/abs-1903-01620} explored a similar semantics as ours, but for Logistic Regression
models. In this paper, we focus on efficient algorithms for nearest neighbor classifiers.}


\stitle{Data Cleaning and Analytics-Driven Cleaning.} The research on data cleaning (DL) has been thriving for many years. Many data cleaning works focus on performing standalone cleaning without considering how cleaned data is used by downstream analytics. We refer readers to a recent survey on this topic~\cite{DBLP:books/acm/IlyasC19}. 

As data cleaning itself is an expensive process that usually needs human involvement eventually (e.g., to confirm suggested repairs), the DB community is starting to work on analytics-driven cleaning methods. SampleClean~\cite{DBLP:conf/sigmod/WangKFGKM14} targets the problem of answering SQL aggregate queries when the input data is dirty by cleaning a sample of the dirty dataset, and at the same time, providing statistical guarantees on the query results.
ActiveClean~\cite{DBLP:journals/pvldb/KrishnanWWFG16} is an example of cleaning data intelligently for convex ML models that are trained using gradient descent methods.
As discussed before, while both ActiveClean and our proposal assume the use of a human cleaning oracle, they are incomparable as they are targeting different ML models. 
BoostClean~\cite{krishnan2017boostclean} automatically selects from a predefined space of cleaning algorithms, using a hold-out validation set via statistical boosting. We show that our proposal significantly outperforms BoostClean under the same space of candidate repairs.

 \section{Conclusion}\label{sec:conclusion}
 
In this work, we focused on the
problem of understanding the impact
of incomplete information on
training downstream ML models. We present a formal study of this impact by extending the notion of \emph{Certain Answers for Codd tables}, which has been explored by the database research community for decades, into the field of machine learning, by introducing the notion of \emph{Certain Predictions} (CP).
We developed efficient algorithms to analyze the impact via CP primitives, in the context of nearest neighor classifiers.
As an application, we further proposed a novel ``DC for ML'' framework built on top of CP primitives that often significantly outperforms existing techniques in accuracy, with mild manual cleaning effort.



\clearpage

\bibliographystyle{IEEEtran}

\bibliography{publication-abbreviations,ms}  

\begin{thebibliography}{10}
\providecommand{\url}[1]{#1}
\csname url@samestyle\endcsname
\providecommand{\newblock}{\relax}
\providecommand{\bibinfo}[2]{#2}
\providecommand{\BIBentrySTDinterwordspacing}{\spaceskip=0pt\relax}
\providecommand{\BIBentryALTinterwordstretchfactor}{4}
\providecommand{\BIBentryALTinterwordspacing}{\spaceskip=\fontdimen2\font plus
\BIBentryALTinterwordstretchfactor\fontdimen3\font minus
  \fontdimen4\font\relax}
\providecommand{\BIBforeignlanguage}[2]{{%
\expandafter\ifx\csname l@#1\endcsname\relax
\typeout{** WARNING: IEEEtran.bst: No hyphenation pattern has been}%
\typeout{** loaded for the language `#1'. Using the pattern for}%
\typeout{** the default language instead.}%
\else
\language=\csname l@#1\endcsname
\fi
#2}}
\providecommand{\BIBdecl}{\relax}
\BIBdecl

\bibitem{Alicebook}
S.~Abiteboul, R.~Hull, and V.~Vianu, \emph{Foundations of Databases: The
  Logical Level}, 1st~ed.\hskip 1em plus 0.5em minus 0.4em\relax USA:
  Addison-Wesley Longman Publishing Co., Inc., 1995.

\bibitem{PDBBOOK}
\BIBentryALTinterwordspacing
D.~Suciu, D.~Olteanu, C.~Ré, and C.~Koch, ``Probabilistic databases,''
  \emph{Synthesis Lectures on Data Management}, vol.~3, no.~2, pp. 1--180,
  2011. [Online]. Available:
  \url{https://doi.org/10.2200/S00362ED1V01Y201105DTM016}
\BIBentrySTDinterwordspacing

\bibitem{DBLP:conf/pods/ArenasBC99}
M.~Arenas, L.~Bertossi, and J.~Chomicki, ``Consistent query answers in
  inconsistent databases,'' in \emph{Proc. 18th ACM SIGACT-SIGMOD-SIGART Symp.
  on Principles of Database Systems}, 1999, pp. 68--79.

\bibitem{10.1145/2213556.2213588}
\BIBentryALTinterwordspacing
P.~K. Agarwal, A.~Efrat, S.~Sankararaman, and W.~Zhang, ``Nearest-neighbor
  searching under uncertainty,'' in \emph{Proceedings of the 31st ACM
  SIGMOD-SIGACT-SIGAI Symposium on Principles of Database Systems}, ser. PODS
  ’12.\hskip 1em plus 0.5em minus 0.4em\relax New York, NY, USA: Association
  for Computing Machinery, 2012, p. 225–236. [Online]. Available:
  \url{https://doi.org/10.1145/2213556.2213588}
\BIBentrySTDinterwordspacing

\bibitem{10.1145/2955098}
\BIBentryALTinterwordspacing
P.~K. Agarwal, B.~Aronov, S.~Har-Peled, J.~M. Phillips, K.~Yi, and W.~Zhang,
  ``Nearest-neighbor searching under uncertainty ii,'' \emph{ACM Trans.
  Algorithms}, vol.~13, no.~1, Oct. 2016. [Online]. Available:
  \url{https://doi.org/10.1145/2955098}
\BIBentrySTDinterwordspacing

\bibitem{10.5555/1783823.1783863}
H.-P. Kriegel, P.~Kunath, and M.~Renz, ``Probabilistic nearest-neighbor query
  on uncertain objects,'' in \emph{Proceedings of the 12th International
  Conference on Database Systems for Advanced Applications}, ser.
  DASFAA’07.\hskip 1em plus 0.5em minus 0.4em\relax Berlin, Heidelberg:
  Springer-Verlag, 2007, p. 337–348.

\bibitem{krishnan2017boostclean}
S.~Krishnan, M.~J. Franklin, K.~Goldberg, and E.~Wu, ``Boostclean: Automated
  error detection and repair for machine learning,'' \emph{arXiv preprint
  arXiv:1711.01299}, 2017.

\bibitem{DBLP:journals/pvldb/KrishnanWWFG16}
S.~Krishnan, J.~Wang, E.~Wu, M.~J. Franklin, and K.~Goldberg, ``Activeclean:
  Interactive data cleaning for statistical modeling,'' \emph{Proc. VLDB
  Endowment}, vol.~9, no.~12, pp. 948--959, 2016.

\bibitem{li2019cleanml}
P.~Li, X.~Rao, J.~Blase, Y.~Zhang, X.~Chu, and C.~Zhang, ``Cleanml: A benchmark
  for joint data cleaning and machine learning [experiments and analysis],''
  \emph{arXiv preprint arXiv:1904.09483}, 2019.

\bibitem{Chen2015SequentialIM}
Y.~Chen, H.~Hassani, S., A.~Karbasi, and A.~Krause, ``Sequential information
  maximization: When is greedy near-optimal?'' in \emph{Conference on Learning
  Theory}, 2015.

\bibitem{rekatsinas2017holoclean}
T.~Rekatsinas, X.~Chu, I.~F. Ilyas, and C.~R{\'e}, ``Holoclean: Holistic data
  repairs with probabilistic inference,'' \emph{arXiv preprint
  arXiv:1702.00820}, 2017.

\bibitem{ChuKATARA}
X.~Chu, J.~Morcos, I.~F. Ilyas, M.~Ouzzani, P.~Papotti, N.~Tang, and Y.~Ye,
  ``{KATARA:} {A} data cleaning system powered by knowledge bases and
  crowdsourcing,'' in \emph{Proc. ACM SIGMOD Int. Conf. on Management of Data},
  2015, pp. 1247--1261.

\bibitem{DBLP:books/acm/IlyasC19}
\BIBentryALTinterwordspacing
I.~F. Ilyas and X.~Chu, \emph{Data Cleaning}.\hskip 1em plus 0.5em minus
  0.4em\relax {ACM}, 2019. [Online]. Available:
  \url{https://doi.org/10.1145/3310205}
\BIBentrySTDinterwordspacing

\bibitem{krishnan2016activeclean}
S.~Krishnan, J.~Wang, E.~Wu, M.~J. Franklin, and K.~Goldberg, ``Activeclean:
  Interactive data cleaning for statistical modeling,'' \emph{Proceedings of
  the VLDB Endowment}, vol.~9, no.~12, pp. 948--959, 2016.

\bibitem{DBLP:journals/ior/KoLQ95}
\BIBentryALTinterwordspacing
C.~Wa~Ko, J.~Lee, and M.~Queyranne, ``An exact algorithm for maximum entropy
  sampling,'' \emph{Oper. Res.}, vol.~43, no.~4, pp. 684--691, 1995. [Online].
  Available: \url{https://doi.org/10.1287/opre.43.4.684}
\BIBentrySTDinterwordspacing

\bibitem{babyproduct}
S.~Das, A.~Doan, P.~S. G.~C., C.~Gokhale, P.~Konda, Y.~Govind, and D.~Paulsen,
  ``The magellan data repository,''
  \url{https://sites.google.com/site/anhaidgroup/projects/data}.

\bibitem{supreme}
J.~S. Simonoff, \emph{Analyzing categorical data}.\hskip 1em plus 0.5em minus
  0.4em\relax Springer Science \& Business Media, 2013.

\bibitem{DELVE}
C.~E. Rasmussen, R.~M. Neal, G.~E. Hinton, D.~van Camp, M.~Revow,
  Z.~Ghahramani, R.~Kustra, and R.~Tibshirani, ``The delve manual,'' \emph{URL
  http://www. cs. toronto. edu/\~{} delve}, 1996.

\bibitem{rubin1976inference}
D.~B. Rubin, ``Inference and missing data,'' \emph{Biometrika}, vol.~63, no.~3,
  pp. 581--592, 1976.

\bibitem{DBLP:conf/icdt/LopatenkoB07}
A.~Lopatenko and L.~E. Bertossi, ``Complexity of consistent query answering in
  databases under cardinality-based and incremental repair semantics,'' in
  \emph{Proc. 11th Int. Conf. on Database Theory}, 2007, pp. 179--193.

\bibitem{DBLP:series/synthesis/2011Bertossi}
L.~E. Bertossi, \emph{Database Repairing and Consistent Query Answering}.\hskip
  1em plus 0.5em minus 0.4em\relax Morgan {\&} Claypool Publishers, 2011.

\bibitem{10.1007/s00521-009-0295-6}
\BIBentryALTinterwordspacing
P.~J. Garc\'{\i}a-Laencina, J.-L. Sancho-G\'{o}mez, and A.~R. Figueiras-Vidal,
  ``Pattern classification with missing data: A review,'' \emph{Neural Comput.
  Appl.}, vol.~19, no.~2, p. 263–282, Mar. 2010. [Online]. Available:
  \url{https://doi.org/10.1007/s00521-009-0295-6}
\BIBentrySTDinterwordspacing

\bibitem{10.2307/2291635}
\BIBentryALTinterwordspacing
D.~B. Rubin, ``Multiple imputation after 18+ years,'' \emph{Journal of the
  American Statistical Association}, vol.~91, no. 434, pp. 473--489, 1996.
  [Online]. Available: \url{http://www.jstor.org/stable/2291635}
\BIBentrySTDinterwordspacing

\bibitem{DBLP:journals/corr/abs-1903-01620}
\BIBentryALTinterwordspacing
P.~Khosravi, Y.~Liang, Y.~Choi, and G.~V. den Broeck, ``What to expect of
  classifiers? reasoning about logistic regression with missing features,''
  \emph{CoRR}, vol. abs/1903.01620, 2019. [Online]. Available:
  \url{http://arxiv.org/abs/1903.01620}
\BIBentrySTDinterwordspacing

\bibitem{DBLP:conf/sigmod/WangKFGKM14}
J.~Wang, S.~Krishnan, M.~J. Franklin, K.~Goldberg, T.~Kraska, and T.~Milo, ``A
  sample-and-clean framework for fast and accurate query processing on dirty
  data,'' in \emph{Proc. ACM SIGMOD Int. Conf. on Management of Data}, 2014,
  pp. 469--480.

\end{thebibliography}

\clearpage

\appendix

\renewcommand\thefigure{\thesection.\arabic{figure}}
\renewcommand\thetheorem{\thesection.\arabic{theorem}}
\renewcommand\thelemma{\thesection.\arabic{lemma}}
\renewcommand\thealgorithm{\thesection.\arabic{algorithm}}
\renewcommand\theequation{\thesection.\arabic{equation}}
\renewcommand\theexample{\thesection.\arabic{example}}

\section{The SS Algorithm for Q2}

\setcounter{figure}{0}
\setcounter{theorem}{0}
\setcounter{lemma}{0}
\setcounter{algorithm}{0}
\setcounter{equation}{0}
\setcounter{example}{0}

\begin{figure}
    \centering
    \includegraphics[width=0.8\textwidth]{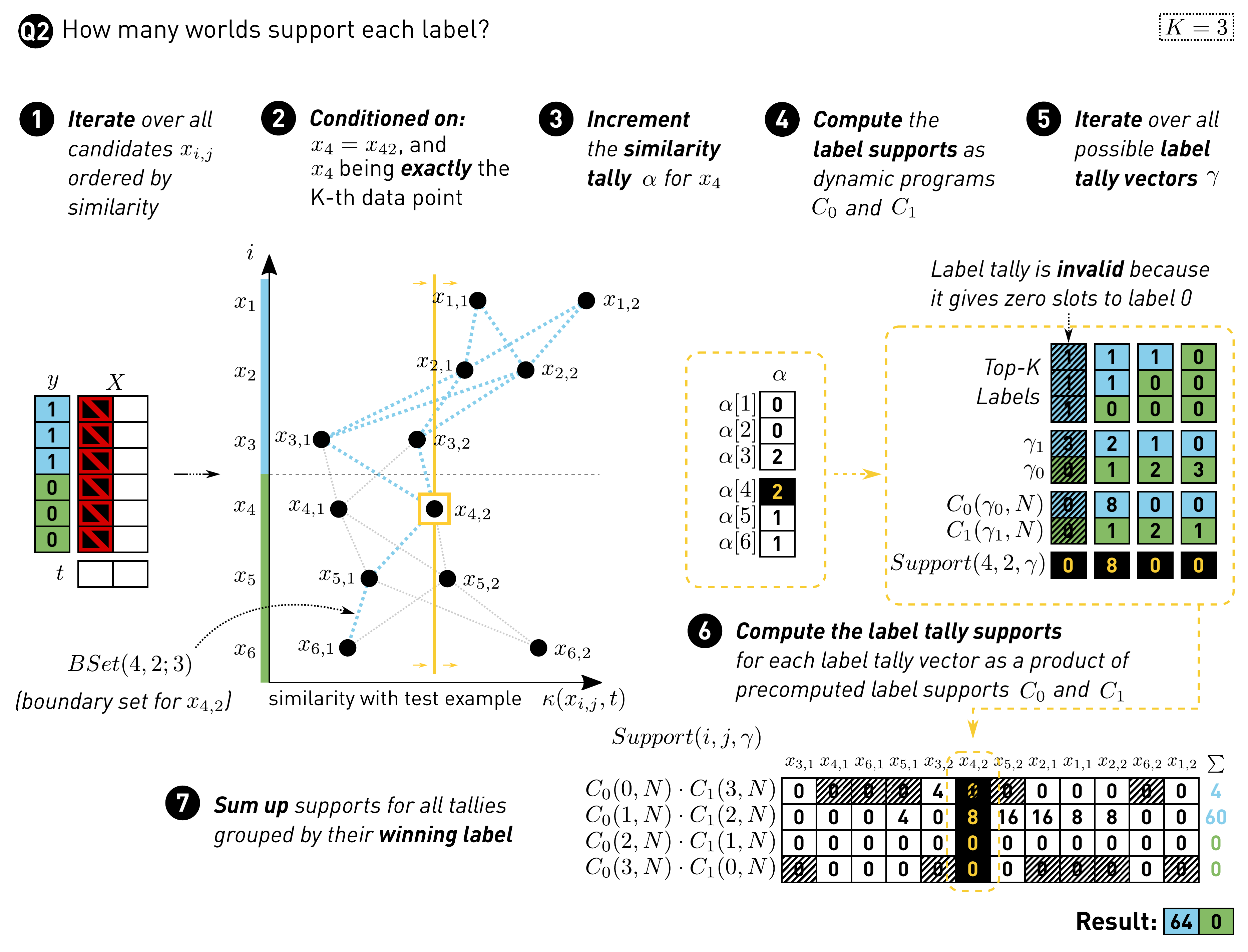}
    \caption{Illustration of SS when $K=3$ for $Q2$.}
    \label{fig:alg-sort-count-k3}
\end{figure}

\stitle{Algorithm Outline.} The SS algorithm answers the counting query according to the following expression:
\begin{equation} \label{eq:ss-outer-sums}
Q2(\mathcal{D}, t, l) =
   \sum_{i \in [N]} \sum_{j \in [M]}
   \sum_{\gamma \in \Gamma}
     \mathbb{I}[l = \argmax(\gamma)] 
     \cdot    
     Support (i, j, \gamma)
\end{equation}
Here, we iterate over all $i=1,...,N$ candidate sets $\mathcal{C}_i$, all their $j=1,...,M$ candidates $x_{i,j} \in \mathcal{C}_i$, and all possible tally vectors $\gamma \in \Gamma$. In each iteration we check if the winning label induced by tally vector $\gamma$ is the label $l$ which we are querying for. If yes, we include the label tally support into the count. An example of this iteration is depicted in the lower right table in Figure~\ref{fig:alg-min-max-k3}.

The \emph{label tally support} $Support(i, j, \gamma)$ is the number of possible worlds from the boundary set $BSet(i,j;K)$ where the tally of the labels in the top-$K$ is exactly $\gamma$. The boundary set $BSet(i,j;K)$ is the set of possible worlds where the value of example $x_i$ is taken to be $x_i=x_{i,j}$ \emph{and} $x_i$ is exactly the $K$-th most similar example to the test example $t$. Given a label tally $\gamma$, its label tally support is computed as:
\begin{equation} \label{eq:alg-ss-support:a}
    Support(i, j, \gamma) = \prod\nolimits_{l \in \mathcal{Y}}
C_{l}^{i, j}(\gamma_l, N)
\end{equation}
Here, $C_l^{i,j}(\gamma_l, N)$ is the \emph{label support} for label $l$, which is the number of possible worlds from the \emph{boundary set} $BSet(i,j;K)$ having exactly $\gamma_l$ examples in the top-$K$. Computing the label support is performed by using the following recursive structure:
\begin{equation*}
C_l^{i, j}(c,n) =
\begin{cases}
C^{i, j}_l(c, n-1), \quad \text{if} \ y_n \neq l, \\
C^{i, j}_l(c-1, n-1), \quad \text{if} \ x_n=x_i, \text{otherwise} \\
\alpha_{i, j}[n] \cdot C^{i, j}_l(c, n-1) + (M-\alpha_{i, j}[n]) \cdot C^{i, j}_l(c-1,n-1).
\end{cases}    
\end{equation*}
This recursion can be computed as a dynamic program, defined over $c\in\{0...K\}$ and $n\in\{1...N\}$. Its boundary conditions are $C_l^{i, j}(-1,n)=0$ and $C_l^{i, j}(c,0)=1$. To compute the support, it uses similarity tallies $\alpha$, defined as such:
\[
\alpha_{i, j}[n] = \sum\nolimits_{m=1}^M \mathbb{I}[\kappa(x_{n, m},t) \leq \kappa(x_{i, j}, t)].
\] 

\subsection{Proof of Correctness}

\begin{theorem}\label{theorem:alg-ss-correct}
The SS algorithm correctly answers $Q2(\mathcal{D}, t, l)$.
\end{theorem}

\begin{proof}

The SS algorithm aims to solve a counting problem by using a technique of partitioning a set and then computing the sizes of the relevant partitions. To prove its correctness we need to: (1) argue that the partitioning procedure is valid and produces disjoint subsets of the original set; and (2) argue that the size of the subset is computed correctly.

To prove the validity of the partitioning method, we start off by reviewing how a brute-force approach would answer the same query:
\begin{equation*}
Q2(\mathcal{D}, t, l) = \sum_{D \in \mathcal{D}} \mathbb{I}[ \mathcal{A}_D(t)=l ]
\end{equation*}

When we partition the sum over all possible worlds into boundary sets $BSet(i,j;K)$ for each $i\in{1...N}$ and $j\in{1...M}$, we obtain the following expression:

\begin{equation*}
Q2(\mathcal{D}, t, l) =
    \sum_{i \in [N]} \sum_{j \in [M]}
    \sum_{D \in BSet(i,j;K)}
        \mathbb{I}[ \mathcal{A}_D(t)=l ]
\end{equation*}

As we mentioned, a boundary set is the set of the possible worlds where $x_i=x_{i,j}$ and $x_i$ is the $K$-th most similar data example to $t$. Since every possible world selects just one candidate per candidate set, for every $i\in\{1...N\}$, the possible world where $x_i=x_{i,j}$ is always different from the possible world where $x_i=x_{i,j'}$, for every $j,j' \in \{1...M\}$ such that $j \neq j'$. Furthermore, every possible world induces a fixed ordering of data examples based on their similarity to $t$. Therefore, any possible worlds where $x_i$ occupies the $K$-th position in that ordering is different from the possible world where it occupies any other position. Thus, we can conclude that all boundary sets $BSet(i,j,K)$ are distinct for all distinct $i$ and $j$.

Given that we are dealing with the $K$-NN algorithm, since each possible world $D$ induces a fixed set of top-$K$ examples, consequently it induces a fixed top-$K$ label tally $\gamma^D$. Since only one label tally of all the possible ones will be correct one, we can rewrite the inner sum as:

\begin{equation*}
Q2(\mathcal{D}, t, l) =
    \sum_{i \in [N]} \sum_{j \in [M]}
    \sum_{D \in BSet(i,j;K)}
    \sum_{\gamma \in \Gamma}
        \mathbb{I}[ 
            \gamma^D = \gamma
        ]
        \mathbb{I}[
            l = \arg\max \gamma
        ]
\end{equation*}

Since in the above expression, the $\gamma$ is independent from $D$, we can reorganize the sums as such:

\begin{equation*}
Q2(\mathcal{D}, t, l) =
    \sum_{i \in [N]} \sum_{j \in [M]}
    \sum_{\gamma \in \Gamma}
        \mathbb{I}[
            l = \arg\max \gamma
        ]
    \sum_{D \in BSet(i,j;K)}
        \mathbb{I}[ 
            \gamma^D = \gamma
        ]
\end{equation*}

We can notice that the innermost sum is equivalent to the definition of a label tally support, which means we can replace it as such:

\begin{equation*}
Q2(\mathcal{D}, t, l) =
   \sum_{i \in [N]} \sum_{j \in [M]}
   \sum_{\gamma \in \Gamma}
     \mathbb{I}[l = \argmax(\gamma)] 
     \cdot    
     Support (i, j, \gamma)
\end{equation*}

Assuming that the label tally support $Support(i,j, \gamma)$ is computed correctly, as shown in Section~\ref{sec:alg:ss:two-building-blocks}, we can conclude that both the partitioning and the partition size computation problems are solved correctly, hence proving our original claim.

\end{proof}

\subsection{Optimization Using Divide and Conquer}

\stitle{Algorithm Outline.} This version of the algorithm is almost identical to the original SS algorithm described previously, except for the way it computes the label support. Namely, in the original algorithm we were using the dynamic program $C_l^{i,j}(c,n)$ to return the number of possible worlds in the boundary set $BSet(i,j;K)$ that support having exactly $c$ examples in the top-$K$. Here, the parameter $n\in\{1...N\}$ denoted that we were only considering the subset of candidate sets $\mathcal{C}_i$ where $i \in \{1...n\}$.

If we observe Algorithm~\ref{alg:sort-count:dp:binary}, we can see that the dynamic program $C_l^{i,j}(c,n)$ is re-computed in every iteration of the outer loop. However, at the same time we can see that the similarity tally $\alpha$, which is used to compute the dynamic program, gets only one of its elements updated. To take advantage of that, we apply a divide-and-conquer strategy and redefine the recurrence relation as a tree structure:

\begin{equation}\label{eq:recurrence-dc}
\begin{split}
    &T_l^{i,j}(c, a, b) = \sum_{k=0}^c T_l^{i,j}(k, a, m) \cdot T_l^{i,j}(c-k, a+1, b), \\
&\text{where }m = \lfloor\frac{a + b}{2} \rfloor.
\end{split}
\end{equation}

To efficiently maintain the dynamic program $T_l^{i,j}$ across iterations over $(i,j)$, we organize it in a binary tree structure. Each node, denoted as $n_{a,b}$, contains a list of values of $T_l^{i,j}(c,a,b)$ for all $c \in \{0...K\}$. Its two children are $n_{a,m}$ and $n_{m+1,b}$ where $m = \lfloor\frac{a + b}{2} \rfloor$. The leaves are nodes $n_{a,a}$ with both indices equal, which get evaluated according to the following base conditions:

\begin{enumerate}
    \item $T_l^{i,j}(c, a, a)=1$, if $y_a \neq l$; \\
    \emph{\textbf{Rationale:}} Skip examples with label different from $l$.
    
    \item $T_l^{i,j}(0, i, i)=0$ and $T_l^{i,j}(1, i, i)=1$; \\
    \emph{\textbf{Rationale:}} The $i$-th example must be in the top-$K$, unless it got skipped.
    \item $T_l^{i,j}(0, a, a)=\alpha[a]$; \\
    \emph{\textbf{Rationale:}} If the $a$-th example is in the top-$K$, there are $\alpha[a]$ candidates to choose from.
    \item $T_l^{i,j}(1, a, a)=M-\alpha[a]$; \\
    \emph{\textbf{Rationale:}} If the $a$-th example is not in the top-$K$, there are $M-\alpha[a]$ candidates to choose from.
    \item $T_l^{i,j}(c, a, a)=0$, if $c \notin \{0,1\}$ \\
    \emph{\textbf{Rationale:}} Invalid case because an example can either be ($c=1$) or not be ($c=0$) in the top-$K$.
\end{enumerate}

The leaf nodes $n_{a,a}$ of this tree correspond to label support coming from individual data examples. The internal nodes $n_{a,b}$ correspond to the label support computed over all leaves in their respective sub-trees. This corresponds to data examples with index $i \in \{a...b\}$. The root node $n_{1,N}$ contains the label support computed over all data examples.

Since between any two consecutive iterations of $(i,j)$ in Algorithm~\ref{alg:sort-count:dp:binary} we only update the $i$-th element of the similarity tally $\alpha$, we can notice that out of all leaves in our binary tree, only $n_{i,i}$ gets updated. This impacts only $\mathcal{O}(N)$ internal nodes which are direct ancestors to that leaf. If we update only those nodes, we can avoid recomputing the entire dynamic program. The full algorithm is listed in Algorithm~\ref{alg:sort-count:dpdc:binary}.

\begin{algorithm}
\caption{Algorithm \textbf{SS-DC} for Q2 with $K$-NN.} \label{alg:sort-count:dpdc:binary}
\begin{algorithmic}[1]
\REQUIRE {$\mathcal{D}$, incomplete dataset; $t$, target data point.}
\ENSURE {$r$, integer vector, s.t. $r[y]=Q2(\mathcal{D}, t, y), \forall y \in \mathcal{Y}$.}

\STATE $s \gets \mathtt{kernel} (\mathcal{D}, t)$;
\STATE $\alpha \gets \emph{zeros}(|\mathcal{D}|)$;
\STATE $r \gets \emph{zeros}(|\mathcal{Y}|)$;
\FORALL{$l \in \mathcal{Y}$}
        \FORALL{$k \in [K]$}
        \STATE Initialize the tree $T_l(k,1,N)$.
        \ENDFOR
    \ENDFOR
\FORALL{$(i,j) \in \mathtt{argsort}(s)$}
    \STATE $\alpha[i] \gets \alpha[i] + 1$;
    \STATE Update the leaf node $T_{y_i}(c, i, i)$ and its ancestors for all $c\in \{1...K\}$.
    \FORALL{possible valid tally vectors $\gamma \in \Gamma$}
        \STATE $y_p \gets \mathtt{argmax} (\gamma)$;
        \STATE Compute 
        $Support(i, j, \gamma) = \mathbb{I}[\gamma_{y_i} \geq 1] \cdot \prod_{l \in \mathcal{Y}} T_l(\gamma_l, 1, N)$;
        \STATE $r[y_p] \gets r[y_p] + Support(i, j, \gamma)$;
    \ENDFOR

\ENDFOR

\RETURN $r$;
\end{algorithmic}
\end{algorithm}

\stitle{Complexity.} We analyze the complexity of Algorithm~\ref{alg:sort-count:dpdc:binary}:
\begin{itemize}
    \item The sorting procedure requires $\mathcal{O} (N \cdot M \log N \cdot M)$ steps as it sorts all elements of $S$.
    \item The tree initialization can be performed eagerly in $\mathcal{O}(KN)$ time, or lazily in constant amortized time.
    \item The outer loop iterates over $\mathcal{O} (N \cdot M)$ elements.
    \item In each inner iteration, we update $\mathcal{O}(\log N)$ nodes. Each node maintains support values for all $c \in \{1...K\}$ and each one takes $\mathcal{O}(K)$ to recompute. Therefore, the tree update can be performed in $\mathcal{O}(K^2 \log N)$ time.
    \item Furthermore, in each iteration, we iterate over all possible label assignments, which requires $\mathcal{O}\left({|\mathcal{Y}|+K-1 \choose K}\right)$ operations. 
    \item For each label assignment, we need $\mathcal{O}(|\mathcal{Y}|)$ multiplications.
\end{itemize}
This renders the final complexity to be the sum of $\mathcal{O} (N \cdot M \cdot (\log(N\cdot M) +  K^2 \cdot \log N ))$ and $\mathcal{O} (N \cdot M \cdot (|\mathcal{Y}| + {|\mathcal{Y}|+K-1 \choose K}))$. When $|\mathcal{Y}|$ and $K$ are relatively small constants, this reduces to $\mathcal{O} (N \cdot M \cdot \log(N\cdot M))$.

\subsection{Polynomial Time Solution for $|\mathcal{Y}| \geq 1$}

We have seen that the previously described version of the SS algorithm gives an efficient polynomial solution for $Q2$, but only for a relatively small number of classes $|\mathcal{Y}|$. When $|\mathcal{Y}| \gg 1$, the $\mathcal{O} ({|\mathcal{Y}|+K-1 \choose K})$ factor of the complexity starts to dominate. For a very large number of classes (which is the case for example in the popular ImageNet dataset), running this algorithm becomes practically infeasible. In this section we present a solution for $Q2$ which is polynomial in $|\mathcal{Y}|$.

The main source of computational complexity in Algorithm~\ref{alg:sort-count:dpdc:binary} is in the for-loop starting at line 10. Here we iterate over all possible tally vectors $\gamma$, and for each one we compute the label tally support $Support(i,j,\gamma)$ (line 12) and add it to the resulting sum (line 13) which is selected according to the winning label with the largest tally in $\gamma$ (line 11).

The key observation is that, for $l$ to be the winning label, one only needs to ensure that no other label has a larger label tally. In other words, label $l$ will be predicted whenever $\gamma_l > \gamma_{l'}$ for all $l \neq l'$, regardless of the actual tallies of all $l'$. Therefore, we found that we can group all the label tally vectors according to this predicate. To achieve this, we define the following recurrence:

\begin{equation}
    D_{Y,c} (j, k) = \sum_{n=0}^{\min\{c,k\}} T_{Y_j}(n, 1, N) \cdot D_{Y,c} (j+1, k-l).
\end{equation}

Here $Y$ is the list of all labels in $\mathcal{Y} \setminus \{l\}$ and $T_{Y_j}^{i,j}(n, 1, N)$ is the label support for label $Y_j$, as described in the previous section. The semantics of $D_{Y,c} (j,k)$ is the number of possible worlds where the top-$K$ contains at most $k$ examples with labels $l' \in Y_{0...j}$ and no label has tally above $c$. We can see that $D_{Y,c} (j, k)$ can also be computed as a dynamic program with base conditions $D_{Y,c} (|Y|, 0)=1$ and $D_{Y,c} (|Y|, k)=0$ for $k > 0$.

\begin{algorithm}
\caption{Algorithm \textbf{SS-DC-MC} for Q2 with $K$-NN.} \label{alg:sort-count:dpdc:multiclass}
\begin{algorithmic}[1]
\REQUIRE {$\mathcal{D}$, incomplete dataset; $t$, target data point.}
\ENSURE {$r$, integer vector, s.t. $r[y]=Q2(\mathcal{D}, t, y), \forall y \in \mathcal{Y}$.}

\STATE $s \gets \mathtt{kernel} (\mathcal{D}, t)$;
\STATE $\alpha \gets \emph{zeros}(|\mathcal{D}|)$;
\STATE $r \gets \emph{zeros}(|\mathcal{Y}|)$;
\FORALL{$l \in \mathcal{Y}$}
        \FORALL{$k \in [K]$}
        \STATE Initialize the tree $T_l(k,1,N)$.
        \ENDFOR
    \ENDFOR
\FORALL{$(i,j) \in \mathtt{argsort}(s)$}
    \STATE $\alpha[i] \gets \alpha[i] + 1$;
    \STATE Update the leaf node $T_{y_i}(c, i, i)$ and its ancestors for all $c\in \{1...K\}$.
    \FORALL{$l \in \mathcal{Y}$}
        \STATE $Y \gets [\mathcal{Y} \setminus \{l\}]$;
        \FORALL{$c \in \{1...K\}$}
            \STATE Compute $D_{Y,c} (|\mathcal{Y}|-1, K-c-1)$ using dynamic programming;
            \STATE $r[y_p] \gets r[y_p] + T_l(c, 1, N) \cdot D_{Y,c} (|\mathcal{Y}|-1, K-c-1)$;
        \ENDFOR
    \ENDFOR

\ENDFOR

\RETURN $r$;
\end{algorithmic}
\end{algorithm}

In terms of performance, the complexity of Algorithm~\ref{alg:sort-count:dpdc:multiclass}, compared to Algorithm~\ref{alg:sort-count:dpdc:binary} has one more major source of time complexity, which is the computation of the dynamic program $D_{Y,c}$ which takes $\mathcal{O} ( |\mathcal{Y}| \cdot K^2)$ time. Since the for loops in lines $10$ and $12$ take $O(|\mathcal{Y}|)$ and $O(K)$ time respectively, the overall complexity of the algorithm becomes $O(M \cdot N \cdot (\log (M \cdot N) + K^2 \log N + |\mathcal{Y}|^2 \cdot K^3 ) $.

\section{The MM Algorithm for Q1}

\setcounter{figure}{0}
\setcounter{theorem}{0}
\setcounter{lemma}{0}
\setcounter{algorithm}{0}
\setcounter{equation}{0}
\setcounter{example}{0}

\stitle{Algorithm Outline.} We are given an incomplete dataset $\mathcal{D} = \{ (\mathcal{C}_i, y_i): i=1,...,N\}$, a test data point $t \in \mathcal{X}$ and a class label $l \in \mathcal{Y}$. The MM algorithm answers the checking query $Q1(\mathcal{D}, t, l)$ for $K$-NN with similarity kernel $\kappa$ by constructing the $l$-expreme possible world $E_{l, \mathcal{D}}$ defined as:

\begin{align} \label{eq:lextreme-def}
    \begin{split}
    E_{l, \mathcal{D}} &= \{ (M_i, y_i) : (\mathcal{C}_i, y_i) \in \mathcal{D} \} \\
    M_i &= \begin{cases}
        \underset{x_{i,j} \in \mathcal{C}_i}{\arg\max} \ \kappa(x_{i,j}, t), & \text{if} \ y_i=l, \\
     \underset{x_{i,j} \in \mathcal{C}_i}{\arg\min} \ \kappa(x_{i,j}, t), & \text{otherwise}
    \end{cases}
    \end{split}
\end{align}

The answer to $Q1(\mathcal{D}, t, l)$ is obtained by checking if: (1) $K$-NN trained over $E_{l,\mathcal{D}}$ predicts $l$, and (2) for all other labels $l' \in \mathcal{Y} \setminus \{l\}$, $K$-NN trained over $E_{l',\mathcal{D}}$ does not predict $l$. Figure~\ref{fig:alg-min-max-k3} depicts this algorithm for an example scenario.


\begin{figure}
    \centering
    \includegraphics[width=0.7\textwidth]{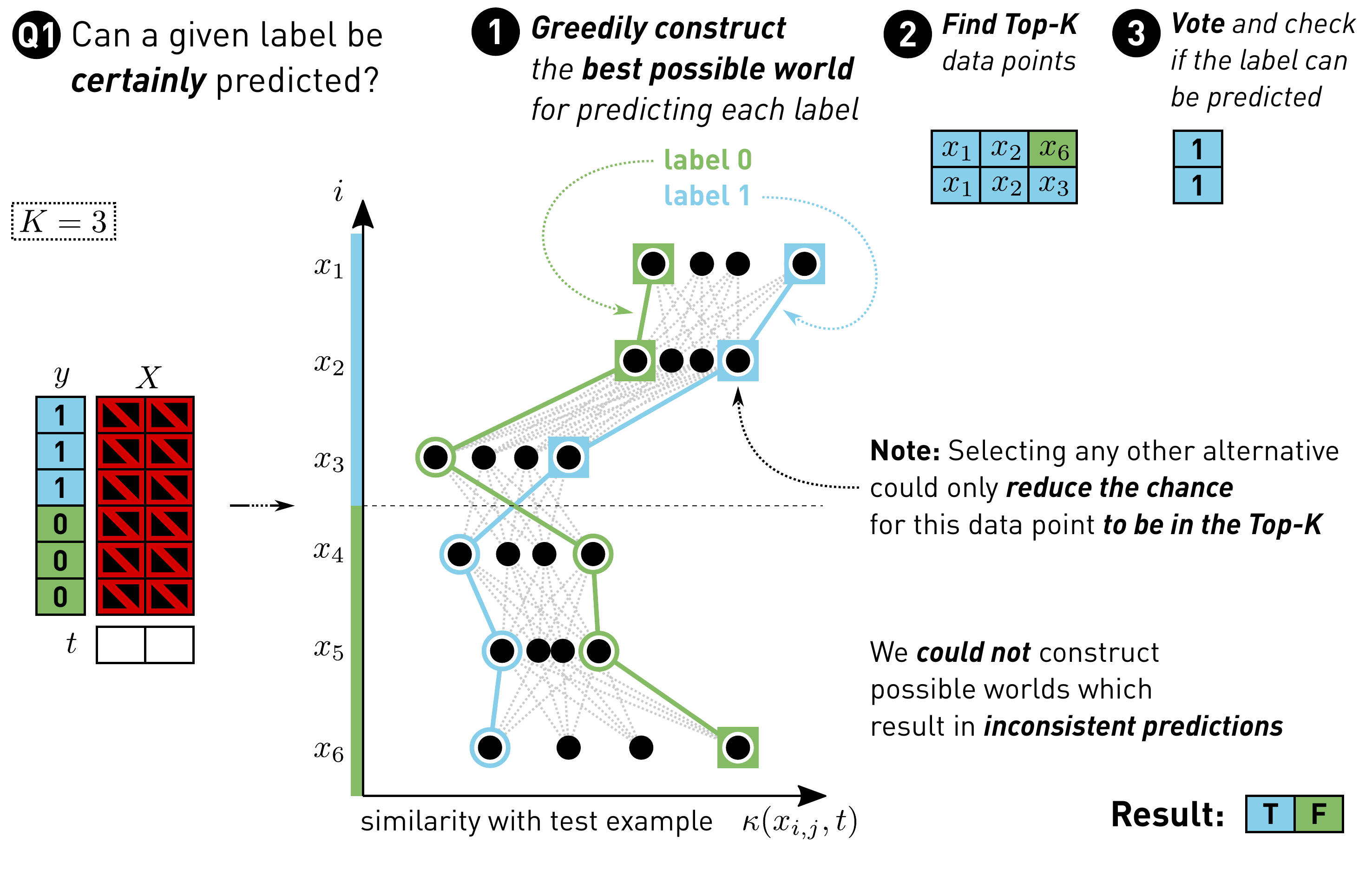}
    \caption{Illustration of MM when $K=3$ for $Q1$.}
    \label{fig:alg-min-max-k3}
\end{figure}

\begin{example}
In Figure~\ref{fig:alg-min-max-k3} we can see an example scenario illustrating the MM algorithm for $K=3$. On the left, we have an incomplete dataset with $N=6$ examples, each with $M=4$ candidates. We construct $l$-extreme worlds for both $l=0$ and $l=1$, by picking the candidate with maximal similarity when $y_i=l$ and the candidate with minimal similarity when $y_i \neq l$. We can see visually that any other choice of candidate could not reduce the odds of $l$ being predicted. In this scenario, we can see that both $l$-extreme worlds predict label $1$, which means that we can conclude that label $1$ can be \emph{certainly predicted}.
\end{example}

\subsection{Proof of Correctness}

\begin{lemma}\label{lemma:alg-mm-relation}

Let $D^{(1)}, D^{(2)} \in \mathcal{I}_\mathcal{D}$ be two possible worlds generated from an incomplete dataset $\mathcal{D}$. Given a test example $t \in \mathcal{X}$ and label $l\in \mathcal{Y}$ where $|\mathcal{Y}|=2$, let $R_{t,l}$ be a partial ordering relation defined as such:
\[
R_{t,l} \left( D^{(1)}, D^{(2)} \right) \ := \
    \bigwedge_{i=1}^N
        \left( y_i= l \ \land \ \kappa(x_i^{(1)},t) \leq \kappa(x_i^{(2)},t) \right) \lor
        \left( y_i \neq l \ \land \ \kappa(x_i^{(1)},t) \geq \kappa(x_i^{(2)},t) \right)
\]

Then, the following relationship holds:

\[
R_{t, l} \left( D^{(1)}, D^{(2)} \right) \ \implies \
\Big( \left( \mathcal{A}_{D^{(1)}}(t)=l \right) \implies \left( \mathcal{A}_{D^{(2)}}(t)=l \right) \Big)
\]

\end{lemma}

\begin{proof}

We will prove this by contradiction. Consider the case when $\mathcal{A}_{D^{(1)}}(t)=l$ and $\mathcal{A}_{D^{(2)}}(t) \neq l$, that is, possible world $D^{(1)}$ predicts label $l$ but possible world $D^{(2)}$ predicts some other label $l' \neq l$. That means that in the top-$K$ induced by $D^{(2)}$ has to be at least one more data point with label $l'$ than in the top-$K$ induced by $D^{(1)}$. Is it possible for the premise to be true?

The similarities $\kappa(x_i^{(1)},t)$ and $\kappa(x_i^{(2)},t)$ cannot all be equal because that would represent equal possible worlds and that would trivially contradict with the premise since the labels predicted by equal possible worlds cannot differ. Therefore, at least one of the $i=1,...,N$ inequalities has to be strict. There, we distinguish three possible cases with respect to the class label $y_i$ of the $i$-th example:
\begin{enumerate}
    \item $y_i = l$: This means that the similarity of a data point coming from $D^{(2)}$ is higher than the one coming from $D^{(1)}$. However, this could only elevate that data point in the similarity-based ordering and could only increase the number of data points with label $l$ in the top-$K$. Since $\mathcal{A}_{D^{(1)}}(t)=l$, the prediction of $D^{(2)}$ could not be different, which is \textbf{contradicts} the premise.
    
    \item $y_i = l' \neq l$ and $|\mathcal{Y}|=2$: We have a data point with a label different from $l$ with lowered similarity, which means it can only drop in the ordering. This can not cause an increase of data points with label $l'$ in the top-$K$, which again \textbf{contradicts} the premise.
    
    \item $y_i = l' \neq l$ and $|\mathcal{Y}| \geq 2$: Here we again have a data point with label $l'$ with lowered similarity. However, if that data point were to drop out of the top-$K$, it would have been possible for a data point with a third label $l'' \notin \{l, l'\}$ to enter the top-$K$ and potentially tip the voting balance in factor of this third label $l''$ (assuming there are enough instances of that label in the top-$K$ already). This would \textbf{not contradict} the premise. However, since the lemma is defined for $|\mathcal{Y}|=2$, this third case can actually never occur.
\end{enumerate}

Finally, for $|\mathcal{Y}|=2$, we can conclude that our proof by contradiction is complete.

\end{proof}

\begin{lemma}\label{lemma:alg-mm-label-prediction}
Let $E_{l,\mathcal{D}}$ be the $l$-extreme world defined in Equation~\ref{eq:lextreme-def}. Then, the $K$-NN algorithm trained over $E_{l,\mathcal{D}}$ will predict label $l$ if and only if there exists a possible world $D \in \mathcal{I}_\mathcal{D}$ that will predict label $l$.
\end{lemma}

\begin{proof}

We consider the following two cases:
\begin{enumerate}
    \item $\mathcal{A}_{E_{l,\mathcal{D}}} (t) = l$: Since $E_{l,\mathcal{D}} \in \mathcal{I}_\mathcal{D}$, the successful prediction of label $l$ represents a trivial proof of the existence of a possible world that predicts $l$.
    
    \item $\mathcal{A}_{E_{l,\mathcal{D}}} (t) \neq l$: We can see that $E_{l,\mathcal{D}}$ is unique because it is constructed by taking from each candidate set $\mathcal{C}_i$ the minimal/maximal element, which itself is always unique (resulting from the problem setup laid out in Section~\ref{sec:alg:ss:two-building-blocks}). Consequently, the relation $R_{t,l}(D, E_{l,\mathcal{D}})$ holds for every $D \in \mathcal{I}_{\mathcal{D}}$. Given Lemma~\ref{lemma:alg-mm-relation}, we can say that if there exists any $D \in \mathcal{I}_{\mathcal{D}}$ that will predict $l$, then it is impossible for $E_{l,\mathcal{D}}$ to not predict $l$. Conversely, we can conclude that if $E_{l,\mathcal{D}}$ does not predict $l$, then no other possible world can predict $l$ either.
\end{enumerate}

\end{proof}

\begin{theorem}\label{theorem:alg-mm-correct}
The MM algorithm correctly answers $Q1(\mathcal{D}, t, l)$.
\end{theorem}

\begin{proof}

The MM algorithm simply constructs the $l'$-extreme world $E_{l',\mathcal{D}}$ for each label $l' \in \mathcal{Y}$ and runs $K$-NN over it to check if it will predict $l'$. Given Lemma~\ref{lemma:alg-mm-label-prediction}, we can conclude that this test is sufficient to check if there exists a possible world that can predict label $l'$. Then, the algorithm simply checks if $l$ is the only label that can be predicted. We can trivially accept that this always gives the correct answer given that it is an exhaustive approach.

\end{proof}

\section{The CPClean Algorithm}

\setcounter{figure}{0}
\setcounter{theorem}{0}
\setcounter{lemma}{0}
\setcounter{algorithm}{0}
\setcounter{equation}{0}

\subsection{Theoretical Guarantee}

We begin with the following supplementary results. 
\begin{lemma}\label{Lemma: semi-submodularity}
Let $D_{\text{Opt}}$ be the optimal set of size $t$ described in Corollary~\ref{Corollary: theoretical guarantee of cleaning}. Denote the set of cleaned training data instances of size $T$ by $D_{\pi}$. For $\mathbf{c}_{\text{Opt}_i} \in D_{\text{Opt}}, i=\{1, 2, \dots, t\}$, we have
\begin{equation}
\begin{split}
    &I(\mathcal{A}_{\mathbf{D}}(D_{\text{val}}); \mathbf{c}_{\text{Opt}_j}|\mathbf{c}_{\text{Opt}_{j-1}}, \dots, \mathbf{c}_{\text{Opt}_{1}}, D_{\pi})\\
    &\leq \theta \min\{ \log |\mathcal{Y}|, \log m\} I(\mathcal{A}_{\mathbf{D}}(D_{\text{val}}); \mathbf{c}_{\pi_{t+1}}|D_{\pi}).
\end{split}
\end{equation}
where $\theta = \frac{1}{\max_{\mathbf{v}\in D_{\text{train}}}I(\mathcal{A}_{\mathbf{D}}(D_{\text{val}}); \mathbf{v})}$.
\end{lemma}

\begin{proof}
We first start with the following inequality:
{\small
\begin{equation}\label{Equation: gamma I}
\begin{split}
    &I(\mathcal{A}_{\mathbf{D}}(D_{\text{val}}); \mathbf{c}_{\text{Opt}_{j}}|\mathbf{c}_{\text{Opt}_{j-1}}, ...,\mathbf{c}_{\text{Opt}_{1}}, D_{\pi})\\
    &=\frac{I(\mathcal{A}_{\mathbf{D}}(D_{\text{val}}); \mathbf{c}_{\text{Opt}_{j}}|\mathbf{c}_{\text{Opt}_{j-1}}, ...,\mathbf{c}_{\text{Opt}_{1}}, D_{\pi})}{\max_{\mathbf{v}}I(\mathcal{A}_{\mathbf{D}}(D_{\text{val}}); \mathbf{v}|D_{\pi})}\max_{\mathbf{v}}I(\mathcal{A}_{\mathbf{D}}(D_{\text{val}}); \mathbf{v}|D_{\pi})\\
    &\leq \frac{\min \{\log |\mathcal{Y}|, \log m\}}{\max_{\mathbf{v}}I(\mathcal{A}_{\mathbf{D}}(D_{\text{val}}); \mathbf{v}|D_{\pi})}I(\mathcal{A}_{\mathbf{D}}(D_{\text{val}});\mathbf{c}_{\text{Opt}_{j}}|D_{\pi})
\end{split}
\end{equation}}
where the last inequality follows from 
\begin{equation*}
\begin{split}
    &I(\mathcal{A}_{\mathbf{D}}(D_{\text{val}}); \mathbf{c}_{\text{Opt}_{j}}|\mathbf{c}_{\text{Opt}_{j-1}}, ..., \mathbf{c}_{\text{Opt}_{1}}, D_{\pi})\\
    &\hspace{1em}\leq \min \{\mathcal{H}(\mathcal{A}_{\mathbf{D}}(D_{\text{val}})), \mathcal{H}(\mathbf{c}_{\text{Opt}_{j}})\}\\ 
    &\hspace{1em}\leq  \min \{\log |\mathcal{Y}|, \log m\}.
\end{split}
\end{equation*}

Let $\mathbf{v}^*=\argmax_{\mathbf{v}\in D_{\text{train}}} I(\mathcal{A}_{\mathbf{D}}(D_{\text{val}}); \mathbf{v})$. We have
\begin{equation}\label{Equation: MI submodularity}
\begin{split}
{\max_{\mathbf{c}_{\text{Opt}_{j}}}I(\mathcal{A}_{\mathbf{D}}(D_{\text{val}});\mathbf{c}_{\text{Opt}_{j}}|D_{\pi})}&\geq {I(\mathcal{A}_{\mathbf{D}}(D_{\text{val}});\mathbf{v}^*|D_{\pi})}\\
&\geq {I(\mathcal{A}_{\mathbf{D}}(D_{\text{val}});\mathbf{v}^*)}
\end{split}
\end{equation}
where the last inequality follows from the independence of $\mathbf{v}_i$'s in $D_{\text{train}}$. That is,
\begin{equation*}
    \begin{split}
      &I(\mathcal{A}_{\mathbf{D}}(D_{\text{val}}), D_{\pi};\mathbf{v}^*)\\  
      &\hspace{1em} = I(\mathcal{A}_{\mathbf{D}}(D_{\text{val}});\mathbf{v}^*)+ I(D_{\pi}^{[l]};\mathbf{v}^*|\mathcal{A}_{\mathbf{D}}(D_{\text{val}}))\\
      &\hspace{1em} = I(\mathcal{A}_{\mathbf{D}}(D_{\text{val}});\mathbf{v}^*|D_{\pi})+ I(D_{\pi}^{[l]};\mathbf{v}^*).
    \end{split}
\end{equation*}

The independence of $\mathbf{v}_i$'s implies that $I(D_{\pi};\mathbf{v}^*)=0$. Hence (\ref{Equation: MI submodularity}) follows.

In the next step, we let $\theta=\frac{1}{{I(\mathcal{A}_{\mathbf{D}}(D_{\text{val}});\mathbf{v}^*)}}$ where $\mathbf{v}^*=\argmax_{\mathbf{v}\in D_{\text{train}}} {I(\mathcal{A}_{\mathbf{D}}(D_{\text{val}});\mathbf{v})}$. Combining (\ref{Equation: gamma I}) and (\ref{Equation: MI submodularity}), we further have
\begin{equation}\label{Equation: gamma II}
\begin{split}
    &I(\mathcal{A}_{\mathbf{D}}(D_{\text{val}}); \mathbf{c}_{\text{Opt}_{j}}|\mathbf{c}_{\text{Opt}_{j-1}}, ..., \mathbf{c}_{\text{Opt}_{1}}, D_{\pi})\\
    &\hspace{1em}\leq \theta{\min \{\log |\mathcal{Y}|, \log m\}}\max_{\mathbf{v}}I(\mathcal{A}_{\mathbf{D}}(D_{\text{val}}); \mathbf{v}|D_{\pi}).
\end{split}
\end{equation}

We remind that our update rule is simply 
\begin{equation*}
    \mathbf{c}_{\pi_{T+1}} \coloneqq \{ \argmax_{\mathbf{v}} I(\mathcal{A}_{\mathbf{D}}(D_{\text{val}}); \mathbf{v}|D_{\pi})\}.
\end{equation*}

Inserting this into (\ref{Equation: gamma II}) proves the Lemma.
 \end{proof}




We now move to the proof of Corollary~\ref{Corollary: theoretical guarantee of cleaning}.
\begin{proof}
We start by noting that $$\mathcal{H}(\mathcal{A}_{\mathbf{D}}(D_{\text{val}})|D_{\text{Opt}})\geq \mathcal{H}(\mathcal{A}_{\mathbf{D}}(D_{\text{val}})|D_{\text{Opt}}, D_{\pi}).$$ We also note
\begin{equation*}
\begin{split}
    &\mathcal{I}(\mathcal{A}_{\mathbf{D}}(D_{\text{val}}); D_{\text{Opt}})= \mathcal{H}(\mathcal{A}_{\mathbf{D}}(D_{\text{val}}))-\mathcal{H}(\mathcal{A}_{\mathbf{D}}(D_{\text{val}})|D_{\text{Opt}})\\
    &\mathcal{I}(\mathcal{A}_{\mathbf{D}}(D_{\text{val}}); D_{\text{Opt}}, D_{\pi})= \mathcal{H}(\mathcal{A}_{\mathbf{D}}(D_{\text{val}}))\\
    & \hspace{11em}-\mathcal{H}(h\mathcal{A}_{\mathbf{D}}(D_{\text{val}})|D_{\text{Opt}}, D_{\pi}).
\end{split}
\end{equation*}

Hence
\begin{equation}\label{Equation: optima vs. combined I}
    {I}(\mathcal{A}_{\mathbf{D}}(D_{\text{val}}); D_{\text{Opt}})\leq{I}(\mathcal{A}_{\mathbf{D}}(D_{\text{val}}); D_{\text{Opt}}, D_{\pi}).
\end{equation}

We further proceed with (\ref{Equation: optima vs. combined I}) as follows.
\begin{equation}\label{Equation: convergence I}
    \begin{split}
         &{I}(\mathcal{A}_{\mathbf{D}}(D_{\text{val}}); D_{\text{Opt}})\\
         &\hspace{1em} \leq{I}(\mathcal{A}_{\mathbf{D}}(D_{\text{val}}); D_{\text{Opt}}, D_{\pi})\\
        &\hspace{1em} ={I}(\mathcal{A}_{\mathbf{D}}(D_{\text{val}}); D_{\text{Opt}}, D_{\pi}) - {I}(\mathcal{A}_{\mathbf{D}}(D_{\text{val}}); D_{\pi})\\
        &\hspace{2em} + {I}(\mathcal{A}_{\mathbf{D}}(D_{\text{val}}); D_{\pi})\\
        &\hspace{1em}= {I}(\mathcal{A}_{\mathbf{D}}(D_{\text{val}}); D_{\text{Opt}}| D_{\pi})+{I}(\mathcal{A}_{\mathbf{D}}(D_{\text{val}}); D_{\pi})\\
        &\hspace{1em}= \sum_{j=1}^t {I}(\mathcal{A}_{\mathbf{D}}(D_{\text{val}}); \mathbf{c}_{\text{Opt}_{j}}| \mathbf{c}_{\text{Opt}_{j-1}}, \dots, \mathbf{c}_{\text{Opt}_{1}}, D_{\pi})\\
        &\hspace{2em}+{I}(\mathcal{A}_{\mathbf{D}}(D_{\text{val}}); D_{\pi})
    \end{split}
\end{equation}
where the last equality follows from the telescopic sum with $\mathbf{c}_{\text{Opt}_{j}} \in D_{\text{Opt}}$.\\ 

Using Lemma~\ref{Lemma: semi-submodularity}, (\ref{Equation: convergence I}) can be followed by
\begin{equation}\label{Equation: convergence II}
    \begin{split}
        &{I}(\mathcal{A}_{\mathbf{D}}(D_{\text{val}}); D_{\text{Opt}})-{I}(\mathcal{A}_{\mathbf{D}}(D_{\text{val}}); D_{\pi})\\
        &\hspace{1em}\leq \sum_{j} {I}(\mathcal{A}_{\mathbf{D}}(D_{\text{val}}); \mathbf{c}_{\text{Opt}_{j}}| \mathbf{c}_{\text{Opt}_{j-1}}, \dots, \mathbf{c}_{\text{Opt}_{1}}, D_{\pi})\\
        &\hspace{1em}\leq \sum_{j} \theta \min\{\log|\mathcal{Y}|, \log m\} \max_{\mathbf{c}_{\text{Opt}_{j}}} {I}(\mathcal{A}_{\mathbf{D}}(D_{\text{val}}); \mathbf{c}_{\text{Opt}_{j}}| D_{\pi})\\
        &\hspace{1em}\leq \theta \min\{\log|\mathcal{Y}|, \log m\} \sum_{j} {I}(\mathcal{A}_{\mathbf{D}}(D_{\text{val}}); \mathbf{c}_{\pi_{T+1}} | D_{\pi})\\
        &\hspace{1em}\leq t \theta\min\{\log|\mathcal{Y}|, \log m\}  {I}(\mathcal{A}_{\mathbf{D}}(D_{\text{val}}); \mathbf{c}_{\pi_{T+1}} | D_{\pi})\\
        &\hspace{1em}\leq t \theta\min\{\log|\mathcal{Y}|, \log m\}  \big({I}(\mathcal{A}_{\mathbf{D}}(D_{\text{val}}); \mathbf{c}_{\pi_{T+1}}\cup D_{\pi})\\
        &\hspace{2em}- {I}(\mathcal{A}_{\mathbf{D}}(D_{\text{val}}); \mathbf{c}_{\pi_{T+1}}\cup D_{\pi}\big ).
    \end{split}
\end{equation}

We further let $\Delta_{T}={I}(\mathcal{A}_{\mathbf{D}}(D_{\text{val}}); D_{\text{Opt}})-{I}(\mathcal{A}_{\mathbf{D}}(D_{\text{val}}); D_{\pi})$. (\ref{Equation: convergence II}) becomes: 
\begin{equation}\label{Equations: terms delta III}
        \Delta_T \leq t\theta \min\{\log |\mathcal{Y}|, \log m\} (\Delta_T -\Delta_{T+1} ).
\end{equation}

Arranging the terms of (\ref{Equations: terms delta III}), we have
\begin{equation*}
t \theta\min\{\log|\mathcal{Y}|, \log m\} \Delta_{T+1} \leq (t\theta\min\{\log|\mathcal{Y}|, \log m\}-1)\Delta_{T}
\end{equation*}
and hence
\begin{equation}\label{Equation: delta alone}
\begin{split}
    \Delta_{T+1} &\leq \frac{t\theta\min\{\log |\mathcal{Y}|, \log m\}-1}{t\theta\min\{\log |\mathcal{Y}|, \log m\}}\Delta_{T}\\
    &\leq \dots\\
    &\leq  \bigg(\frac{t\min\{\log |\mathcal{Y}|, \log m\}\theta-1}{t\theta\min\{\log |\mathcal{Y}|, \log m\}}\bigg)^T \Delta_{0}.
\end{split}
\end{equation}
Noting
{\small 
\begin{equation*}
\Big(\frac{t\theta\min\{\log|\mathcal{Y}|, \log m\} -1}{t\theta\min\{\log|\mathcal{Y}|, \log m\} }\Big)^l\leq \exp(-l/t\theta\min\{\log|\mathcal{Y}|, \log m\} )
\end{equation*}}
we have
\begin{equation*}\label{Equations: terms delta}
\begin{split}
    \Delta_{T+1} &\leq \exp(-T/t\theta\min\{\log |\mathcal{Y}|, \log m\}) \Delta_{0}\\
    &= \exp(-T/t\theta\min\{\log |\mathcal{Y}|, \log m\}\gamma) {I}(\mathcal{A}_{\mathbf{D}}(D_{\text{val}}); D_{\text{Opt}}).
\end{split}
\end{equation*}

By the definition of $\Delta_{T}$, we therefore have 
\begin{equation}
\begin{split}
    &{I}(\mathcal{A}_{\mathbf{D}}(D_{\text{val}}); D_{\pi})\\
    & \hspace{1em }\geq {I}(\mathcal{A}_{\mathbf{D}}(D_{\text{val}}); D_{\text{Opt}}) (1-e^{-\frac{T}{t\theta \min\{\log|\mathcal{Y}|, \log m\}}}).
    \end{split}
\end{equation}
which proves the Corollary~\ref{Corollary: theoretical guarantee of cleaning}.
\end{proof}

\end{document}